\documentclass{article}

\usepackage{xargs}                     \usepackage[colorinlistoftodos,prependcaption,textsize=tiny]{todonotes}

\usepackage[utf8]{inputenc} 
\usepackage[T1]{fontenc}    
\usepackage[breaklinks=true,colorlinks = true,citecolor=blue,linkcolor=blue]{hyperref}
\usepackage{url}            
\usepackage{booktabs,comment}       
\usepackage{amsfonts,graphics, graphicx}       
\usepackage{nicefrac,amsmath,amsthm, amssymb,mathrsfs,wrapfig}       
\usepackage{microtype,xcolor}      
\usepackage{algorithm,algpseudocode}
\usepackage{soul}
\newcommand{\argmin}{\operatornamewithlimits{argmin}}

\newcommand{\poly}{\operatornamewithlimits{poly}}

\usepackage{amsfonts,amsmath,amsthm,dsfont,comment,mathpazo}

\newenvironment{mylist}[1]{\begin{list}{}{
	\setlength{\leftmargin}{#1}
	\setlength{\rightmargin}{0mm}
	\setlength{\labelsep}{2mm}
	\setlength{\labelwidth}{8mm}
	\setlength{\itemsep}{0mm}}}
	{\end{list}}

\usepackage[margin=1in]{geometry}

\newcommand{\exclude}[1]{}
\newcommand{\dpriv}{\mathsf{DS}_{\textsf{priv}}}
\newcommand{\dsmooth}{\mathsf{DS}_{\textsf{spectral}}}
\newcommand{\dnormal}{\mathsf{DS}}

\newcommand{\jalaj}[1]{}

\newcommand{\removed}[1]{}

\newcommand{\opt}{\mathsf{OPT}}

\newtheorem{myremark}{Remark}
\newenvironment{remark}{\begin{myremark}}{\end{myremark}}

\newtheorem{myexample}{Example}

\newcommand{\brak}[1]{{\left\langle {#1} \right\rangle}}
\newcommand{\set}[1]{\left\{{#1} \right\}}
\newcommand{\paren}[1]{\left( {#1} \right)}
\newcommand{\sparen}[1]{\left[ {#1} \right]}

\newcommand{{\R}}{\mathbb{R}}
\newcommand{{\I}}{\mathds{1}}

\newtheorem{theorem}{Theorem}
\newtheorem{lemma}{Lemma}

\newtheorem{proposition}[theorem]{Proposition}

\newtheorem{corollary}{Corollary}
\newtheorem{conjecture}{Conjecture}

\newtheorem{definition}{Definition} 

\newcommand{\lemref}[1]{Lemma~\ref{lem:#1}}

\newcommand{\eqnref}[1]{equation~(\ref{eq:#1})}

\newcommand{{\bQ}}{\mathbf{\Psi}}

\newcommand{\cA}{\mathcal{A}}

\newcommand{\cG}{\mathcal{G}}
\newcommand{\cH}{\mathcal{H}}

\newcommand{\cN}{\mathcal{N}}

\newcommand{\citet}{\cite}
\newcommand{\pr}{\mathsf{Pr}}

\newcommand{\ip}[2]{\left\langle #1 , #2\right\rangle}

\newcommand{\norm}[1]{\left\| #1 \right\|}

\newcommand{\tr}[1]{\mathsf{Tr} \left( #1 \right)}

\usepackage{silence}

\newcommand{\papertitle}{A Framework for Private Matrix Analysis}

\title{\papertitle}

\author{Jalaj Upadhyay\thanks{Joined Apple subsequent to the finalization of this work.} \\
Johns Hopkins University\\
email: \texttt{jalaj.upadhyay@apple.com}
\and 
Sarvagya Upadhyay \\
Fujitsu Laboratories of America \\
email: \texttt{supadhyay@fujitsu.com}
}

\date{}
\begin{document}

\maketitle

\pagenumbering{roman}

\begin{abstract}
We study private matrix analysis in the sliding window model where only the last $W$ updates to matrices are considered useful for analysis. We give first efficient $o(W)$ space differentially private algorithms for spectral approximation, principal component analysis, and linear regression. We also initiate and show efficient differentially private algorithms for two important variants of principal component analysis: sparse principal component analysis and non-negative principal component analysis. Prior to our work, no such result was known for sparse and non-negative differentially private principal component analysis even in the static data setting. These algorithms are obtained by identifying sufficient conditions on positive semidefinite matrices formed from streamed matrices. We also show a lower bound on space required to compute low-rank approximation even if the algorithm gives multiplicative approximation and incurs additive error. This follows via reduction to a certain communication complexity problem.
\end{abstract}

\clearpage

\pagenumbering{arabic}

\section{Introduction}
\label{sec:introduction}
{

An $n \times d$ matrix provides a natural structure for encoding information about $n$ data points, each with $d$ features. Such representation manifests itself in many ways, such as financial transactions, asset management, recommendation system, social networks, machine, and learning kernels.  
As a result, a large body of work has focused on matrix analysis and its application in statistical data analysis, scientific computing, and machine learning~\cite{DrineasM17, durlauf1991spectral, lewis1996convex, pettersen2004ucsf, press2007numerical, sang2002predictability, scharf1991statistical, schott2016matrix, strang1973analysis, Woodruff14}.  

There has been a paradigm shift in statistical analysis (including matrix analysis) in the era of big data. Two aspects that have become increasingly important from a technical viewpoint are 
(i) protecting sensitive information and (ii) the increasing frequency with which data is being continuously updated. Instances that illustrate the importance of these two aspects are investment strategies (across a variety of financial assets) in a financial firm. The strategies rely on matrix analysis of financial data that get continuously updated. Most of these strategies make use of ``recent data" as opposed to the entire history. This heuristic is rooted in the empirical observation that recent data are better predictors of the future behavior of assets than older data~\cite{moore2013taste, Tsay2005}, a theme also found in many other applications of matrix analysis as well~\citet{campos2014time, koren2009collaborative, quadrana2018sequence}. Moreover, the strategies are sensitive and have to be kept private. On the contrary, it is well documented that performing statistical analysis, including matrix analysis, accurately can leak private information~\cite{backstrom2007wherefore, dwork2017exposed, cohen2018linear, garfinkel2018understanding, kenthapadi2005simulatable, narayanan2006break, shokri2017membership}.

Accomplishing matrix analysis has been widely studied when the entire historical data is taken into account~\cite{DrineasM17, Woodruff14}. 
However, one cannot utilize these techniques in the more restrictive scenario where a collection of the most recent updates on data is pertinent for analysis. This is even more challenging when one wishes that the analysis also satisfies a privacy guarantee. Known privacy preserving algorithms for matrix analysis give provable guarantees under a robust privacy guarantee known as {\em differential privacy} (see, for example, ~\cite{amin2019differentially, blum2005practical, dwork2014analyze, kapralov2013differentially, mcsherry2009differentially, hardt2014noisy, hardt2012beating,upadhyay2018price, wang2015differentially, zhou2009differential}), but they are not adaptable to above setting. In contrast, the current practical deployment of private algorithms~\cite{ding2017collecting, erlingsson2014rappor, haney2017utility, kenthapadi2017bringing,  thakurta2017learning,thakurta2017emoji} favors using only recent data for a variety of reasons. 

In the view of this, we focus on a rigorous study of privacy-preserving matrix analysis in the {\it sliding window model of privacy}~\cite{bolot2013private, chan2012differentially, upadhyay2019sublinear}. The model is parameterized by the window size, $W$, and assumes that the data arrives in the form of (possibly infinite) stream over time. An analyst is required to perform the analysis only on the $W$ most recent streams of data (usually referred to as a {\em sliding window}) using $o(W)$ space. The privacy guarantee, on the other hand, is for the entire historical data, i.e., even if the data is not in the current window, its privacy should not be compromised.

We demonstrate $o(W)$ space private algorithms for several matrix analysis problems in the sliding window model (see, Table~\ref{tab:results}). A succinct overview of our main contributions are as follows. 
\begin{enumerate}
    \item 
    {
    We show that the {\em spectral histogram} framework used in the non-private setting for $o(W)$ space algorithm~\cite{braverman2018numerical} is too stringent for privacy and resulting algorithms and analysis are not robust to noise required for privacy\footnote{\label{footnote:first}The  online version of potential barrier~\cite{CohenMP16} has been used in the sliding window setting~\cite{braverman2018numerical}; however, such techniques are not useful under privacy constraints. In these techniques, the sampling probability of a row at any time depends on which other rows are sampled previously. These sampling probabilities are not Lipschitz and highly correlated, so it is not clear how to make it private -- simply adding a high variance noise either results in a bound worse than the trivial additive error. Using a rough estimate of leverage score as done in Arora and Upadhyay~\cite{arora2019differentially} also leads to a sub-optimal bounds for matrix analysis.}. That is, we show rigorously that the constraint imposed by spectral histogram only permits sub-optimally accurate algorithms when privacy is a concern (Appendix~\ref{sec:slidingwishart}).
    }
    
    \item We propose  a relaxation of spectral histogram property on a set of positive semidefinite matrices and show that it suffices for private matrix analysis. We call this relaxation the {\it approximate spectral histogram property}. We design an efficient data structure that maintains the approximate spectral histogram property on a set of positive semidefinite matrices while preserving differential privacy (Section~\ref{sec:slidingJL}) under the assumption that entries of the streamed rows are polynomially bounded. 
    
    \item We use approximate spectral histogram property to efficiently and optimally solve several matrix analysis problems privately in the sliding window model. The problems we consider are (i) spectral approximation, (ii) principal component analysis (PCA), (iii) directional variance queries, and (iv) generalized linear regression. 
    We also study {\em constrained PCA}~\cite{cohen2015dimensionality}, and give a sublinear space private algorithm for it. This generalizes many variants of PCA studied in statistical machine learning such as sparse PCA and non-negative PCA~\cite{asteris2014nonnegative, d2005direct, yuan2013truncated, zass2007nonnegative} (see Table~\ref{tab:results}).

    \item Finally, we exhibit limitations of private matrix analysis by giving a lower bound on differentially private algorithm for low-rank approximation in the sliding window model (Appendix~\ref{sec:lower}). 
\end{enumerate}  
}

\begin{table}[t]
\small{
    \centering
    \begin{tabular}{|c|  c|c |c |c|} 
    \hline
    & Privacy & Additive error & Space required & Reference \\ \hline
    $\eta$-spectral approximation & $(\epsilon,\delta)$-DP & $O\left(\frac{r^2 \log^2 (1/\delta)}{\epsilon^2}\right) \I_d$  & $O\left(\frac{r^2d}{\eta}\log W\right)$ & Theorem~\ref{thm:privslidinglra},~\ref{thm:slidingprivapprox} \\ \hline

    PCA & $(\epsilon,\delta)$-DP & $O\left(\frac{\sqrt{kd} \log(1/\delta)}{\epsilon}\right)$ & $O\left(\frac{dk^2}{\eta^3}\log W\right)$ & Theorem~\ref{thm:pcaimprovedspace} \\ \hline
        
    Sparse and Non-negative PCA & $(\epsilon,\delta)$-DP & $O\left(\frac{\sqrt{kd} \log(1/\delta)}{\epsilon}\right)$ & $O\left(\frac{dk^2}{\eta^3}\log W\right)$ & Theorem~\ref{thm:constrainedLRA} \\ \hline

    Squared linear regression & $(\epsilon,\delta)$-DP & $O\left( d \paren{ d + \frac{ \log(1/\delta)}{\epsilon}}\right)$ & $O\left(\frac{d^3}{\eta}\log W\right)$ & Theorem~\ref{thm:regression} \\ \hline

    Directional variance query & $(\epsilon,\delta)$-DP & $O\left(d \paren{ d + \frac{ \log(1/\delta)}{\epsilon}} \right)$ & $O\left(\frac{d^3}{\eta}\log W\right)$ & Theorem~\ref{thm:covariance} \\ \hline
        
        
\end{tabular}
\caption{Results presented in this paper ($W$: window size, $k:$ target rank, $d:$ dimension of streamed row, and $\epsilon,\delta$ are privacy parameters, $\I_d$ is a $d \times d$ identity matrix, $r$: rank of streamed matrix).}
\label{tab:results}
}
\end{table}

Conceptually, approximate spectral histogram property can be viewed as a generalization of {\em subspace embedding property}~\citet{clarkson2017low}. This allows us to use approximate spectral histogram property in the sliding window model in the same way as subspace embedding is employed in the general streaming model~\cite{Woodruff14} even in the context of privacy~\cite{blocki2012johnson, upadhyay2018price}. Given the wide array of applications of subspace embedding, we believe that the notion of approximate spectral histogram will have further applications in the sliding window model of privacy.

One may ask why we need to introduce approximate spectral histogram property in the sliding window model of privacy. We end this section with a discussion on this. Let us consider the spectral approximation of matrices. There is one private algorithm~\citet{blocki2012johnson} which relies on subspace embedding. This algorithm explicitly computes the singular value decomposition of the matrix making it suitable only for static data matrix. Moreover, we cannot revert the effect of the row streamed beyond the current position of the window. 

{
Finally, we cannot just take an off-the-shelf algorithm,  add an appropriately scaled noise matrix to preserve privacy, and get a non-trivial utility guarantee for downstream matrix analysis. For the start, the standard noise mechanism, like the Gaussian mechanism would result in a matrix that is not a positive semidefinite matrix. If we instead use the projection trick of Arora and Upadhyay~\cite{arora2018differentially}, it would incur error that scales with the dimension and would have an inefficient update time. Likewise, adding a noise matrix that is a positive semidefinite matrix would incur error linear in dimension. We explore this in more detail in Appendix~\ref{sec:slidingwishart}. The known space-efficient algorithm~\cite{braverman2018numerical} performs sampling using the leverage score. As a result, the effect of a single row in the matrix formed by this sampling procedure can be arbitrarily large, and consequently, leading to a trivial utility guarantee. {In fact, one can show that the leverage score for a row can change arbitrarily using Meyer's result~\cite{Meyer73}. It is also not clear if we can even use the exponential mechanism to sample rows because it is not clear how to adapt it to a sliding window setting and for the natural score functions, one can construct counterexamples where the sensitivity of the score function is also large. This is in addition to the question regarding the efficiency of the update stage.} 
}

\subsection{Differential privacy under sliding window model}
We consider a matrix $A_W \in \R^{W \times d}$ that is formed incrementally through a stream of $d$-dimensional row vectors $(a_T)_{T\geq 1}$. At the start, the matrix $A_W(0)$ is an all zero matrix and its  state at time $T$ is 
\begin{align}
 A_W(T) := \begin{pmatrix}
        a_1 \\ \vdots \\ a_T \\ 0^{(W-T) \times d}
    \end{pmatrix}~\text{if}~T \leq W, 
    \quad 
    \text{else}~A_W(T):= \begin{pmatrix}
        a_{T-W+1} \\ \vdots \\ a_{T-1}  \\ a_T 
    \end{pmatrix}.
    \label{eq:sliding}
\end{align}

We fix the symbol $W$ to denote the window size and use the notation $A_W$ to denote $A_W(T)$ whenever it is clear from the context.  At any time $T$, we are interested in performing different types of analysis on the matrix $A_W$, such as PCA and its variants, linear regression, etc. 

We now formalize the privacy model. We adhere to the neighboring relation used in many works studying matrix analysis in static setting \citet{blocki2012johnson,hardt2012beating, dwork2014analyze, sheffet2015private}, in online setting~\cite{dwork2014analyze}, and streaming setting~\cite{upadhyay2018price}. To this end, for any $T>0$, consider the following set of $T \times d$ matrices:
\[
\mathsf N:=\set{B \in \R^{T \times d}:\text{  $\exists i \in [T]$ such that $\norm{B[i:]}_2 \leq 1$ and $\norm{B[j:]}_2 =0$ for all $j \neq i$}},
\]
where $B[i:]$ denotes the $i$-th row of the matrix $B$ and $\norm{B[i:]}_2$ denote its Euclidean norm.

In privacy literature, there are two well-studied levels of granularity when the data arrives in an online manner~\cite{arora2018differentially, bolot2013private, chan2011private, chan2012differentially, dwork2010pan, dwork2014algorithmic, mir2011pan, upadhyay2018price, upadhyay2019sublinear}: (i) {\em user-level privacy}, where two streams are neighboring if they differ in a single user's data; and (ii) {\em event-level privacy}, where two streams are neighboring if they differ in one-time epoch. We follow previous works on private analysis in the sliding window model~\citet{bolot2013private,chan2012differentially, upadhyay2019sublinear} and consider event-level privacy as our notion of privacy. 
We say that two streams are {\em neighboring} if, at any time $T>0$, they form matrices $A_T$ and $A_T'$ such that $A_T - A_T' \in \mathsf N$. We now define the privacy notion. 

\begin{definition}
[Differential privacy under sliding window model~\cite{bolot2013private,chan2012differentially, upadhyay2019sublinear}]
\label{defn:differentialprivacy}
For $\epsilon \geq 0, \delta \in [0,1]$, we say a randomized algorithm $R$ with range $\mathsf{Y}$ is $(\epsilon,\delta)$-differentially private in the sliding window model if for any $T >0$, for every two matrices $A_T$ and $A_T'$ formed by neighboring streams, and for all measurable subsets, $\mathsf S \subseteq \mathsf{Y}$,
$$
\Pr[M(A_T) \in \mathsf S] \leq \exp(\epsilon) \Pr[M(A_T') \in \mathsf S] + \delta,
$$ where the probability is taken over the private coin tosses of the algorithm $M$.  
\end{definition}
Note that the privacy guarantee is for the entire stream, i.e., even if the data has expired, its privacy is not lost. However, accuracy is required only for the last $W$ updates. This is in accordance with previous works.

\subsection{Our techniques}
Our main goal is to privately compute $(\eta,\zeta)$-spectral approximation, which is defined as follows. Given parameters $\eta, \zeta \geq 0$ and a matrix $A_W \in \R^{W \times d}$, find a postive semidefinite matrix $S \in \R^{d \times d}$,  such that 
\begin{align*}
(1-\eta) A_W^\top A_W - \zeta \I_d \preceq S \preceq (1+\eta) A_W^\top A_W + \zeta \I_d.
\end{align*}
Here the partial order $C \preceq D$ between symmetric matrices $C$ and $D$ means that $D - C$ is a positive semidefinite matrix, and the parameter $\zeta \geq 0$ is the distortion in the spectrum that we are willing to accept to preserve privacy. Ideally, we would like these parameters to be as small as possible. 

A naive algorithm, $\mathsf A_{\mathsf{priv}}$, for private spectral approximation is to store a set of $w=\min\set{W, T}$ positive semidefinite matrices at any time $T$, where the $i$-th matrix in this set is a sanitized version of the matrix formed by the last $i$ updates.  $\mathsf A_{\mathsf{priv}}$ requires $O(wd^2)$ space. Our main conceptual contribution is a framework for private matrix analysis in the sliding window model with significantly less space and better accuracy. 
To this end, we introduce {\em $\eta$-approximate spectral histogram property} for a set of positive semidefinite matrices and timestamps. We will occasionally refer to such a set as a data structure for the remainder of this section.

Let $A_W$ be the matrix formed by the window $W$. Let  $\widetilde S_i$ denote an $(\frac{\eta}{4},0)$-spectral approximation of $S_i$, i.e., $(1-\eta/4) \widetilde S_i \preceq S_i \preceq (1+\eta/4) \widetilde S_i$.  
Roughly speaking, a data structure $\mathfrak{D}$ satisfies $\eta$-{\em approximate spectral histogram property} (rigorously defined in Definition~\ref{defn:approxsmoothPSD}) if there exists an $\ell=\poly(d,\log W)$, such that $\mathfrak{D}$ consists of $\ell $ timestamps  $ t_{1} < \cdots < t_{\ell}$ and  PSD matrices $ \widetilde S_\ell \preceq  \cdots \preceq \widetilde S_1$ such that 
\begin{equation}\label{eq:etaapproxproperty}
\forall i \in [\ell], (1-\eta)S_i \preceq S_{i+1}; \quad 
\forall i \in [\ell-2], \left(1-\frac{\eta}{2}\right) \widetilde S_i \not\preceq \widetilde S_{i+2}; \quad
\text{and} \quad
t_2 \leq T-W+1 \leq t_1,
\end{equation} 
where $A \not\preceq B$ implies that $B-A$ is not a  positive semidefinite matrix. Note that the above definition relaxes the condition in Braverman et al.~\cite{braverman2018numerical} because we allow a spectral approximation of the corresponding original matrix. This relaxation allows us to deal with the perturbation required for privacy without losing on the accuracy. In other words, the relaxation allows us to be more robust to the noise. 

The first two conditions are required to get the desirable space bound, while the first and third conditions are required to demonstrate the accuracy guarantee. When it is clear from context, we call such a set of matrices as the one satisfying the $\eta$-approximate spectral histogram property. For the ease of presentation in this overview, we focus on the case when the output is produced just once at the end of the stream.

It is important that $\ell$ is small at any time $T$.We first show that, if the rank of the matrix $A_W$ is $r$, then  $\ell=O\paren{\frac{r}{\eta}\log W}$.  By the first two  properties in Equation~\ref{eq:etaapproxproperty}, there is at least one singular value that decreases by a factor of  $(1-\frac{\eta}{2})$ in every successive timestamp. We will later see that our privacy mechanism ensures that the spectrum of any matrix $\widetilde S_i$ is lower bounded by a constant. Since every update has bounded entries, there can be at most $\ell:=O\paren{{r}\log_{1-\frac{\eta}{2}}(W)} = O\paren{\frac{r}{\eta}\log (W)}$ matrices satisfying $\eta$-approximate spectral histogram. The question we answer next is how to maintain such a set of matrices while preserving privacy. 

As mentioned in the introduction, we cannot just take off-the-shelf algorithm~\cite{braverman2018numerical} and  add noise matrix to preserve privacy as well as the utility. To preserve privacy so that we can also get meaningful utility, when a new row $a_{T+1} \in \R^d$ is streamed, we first update $\mathsf{S}_{T}$ to get $\mathsf{S}_{T+1}'$ as follows. We first add a sketch of $a_{T+1}$ to every matrix in $\mathsf S_T$ to obtain an updated set $\mathsf S_T'$. We then privatize $a_{T+1}$ to get a  matrix $ A_{\ell+1}$ and define $\mathsf{S}_{T+1}' :=  \mathsf S_T' \cup  A_{\ell+1}$. For privacy mechanism, we use a variant of Johnson-Lindenstrauss mechanism~\cite{blocki2012johnson} first proposed in Upadhyay~\cite{upadhyay2014differentially} and extended in Sheffet~\cite{sheffet2015private} and  Upadhyay~\cite{upadhyay2018price}. 

Now the set $ \{ S^\top S: S \in \mathsf S_{T+1}'\}$ may not satisfy $\eta$-approximate spectral histogram property. In the next phase, we greedily remove matrices if they do not satisfy any of the desired properties of $\eta$-approximate spectral histogram property (Algorithm~\ref{fig:slidingprivpcpinitial}). In this phase, the most computationally expensive part of the algorithm is to find if PSD ordering is violated. For this step, we can use known PSD testing algorithms~\cite{bakshi2020testing}. 
Such greedy approach is reminiscent of the {\em potential barrier method}~\citet{batson2012twice} to compute spectral sparsification of a $W \times d$ matrix. In the potential barrier method, we remove a large subset of rank-one matrices and show that only storing $\Theta\left(d\eta^{-2}\right)$ rank-one matrices suffices for $\eta$-spectral sparsification. In fact, two key technical features distinguish our method from theirs. In their setting, all PSD matrices are rank-one matrices; whereas we have $W$ positive semidefinite matrices that may have different ranks (not necessarily rank-one). The second crucial point is that we aim to significantly reduce the number of matrices stored for our application. This makes maintaining our data structure much more complicated than the potential barrier method.

In contrast to the non-private algorithm, for the accuracy proof, we now have to deal with the spectral approximation of the streamed matrix along with the perturbation required to preserve privacy. As such we cannot use previous analysis and have to be more careful when dealing with the privatized matrices. One of the subtle reason is that the utility proof of previous analysis relied on the fact that the checkpoints survived previous deletion, where the deletion criterion relies on the exact covariance matrix of the streamed matrix, while in our case, it depends on the spectral approximation of privatized streamed matrix.   

We now give an overview of how $\eta$-approximate spectral histogram property allows us to output a spectral approximation of the matrix $A_W$. Let  $\widetilde S_{1}, \ldots , \widetilde S_{\ell}$ be the set of matrices satisfying $\eta$-approximate spectral histogram property. From the third condition of Equation~\ref{eq:etaapproxproperty}, we have $S_{2} \preceq A_W^\top A_W \preceq S_{1}$. The first condition of Equation~\ref{eq:etaapproxproperty} implies that $(1-\eta) S_{1} \preceq  S_2$. Combined with the fact that $S_1$ and $S_2$ are a $\frac{\eta}{4}$-spectral approximation of $\widetilde S_1$ and $\widetilde S_2$, respectively, this implies that $\widetilde S_1$ is a spectral approximation of $A_W$. The accuracy proof becomes more subtle because of the noise introduced for privacy. We cannot use previous analysis because they do not extend to the perturbation required for privacy.  Intuitively, the additive term in our spectral approximation is due to the noise introduced by employing privacy mechanisms. We show the following:

\begin{theorem} 
[Informal version of Theorem~\ref{thm:slidingprivapprox}]
\label{thm:privslidinginformal}
\label{thm:informalspectralapprox}
Let $A_W \in \R^{W \times d}$ be a rank-$r$ matrix formed by the current window. Then there is an efficient $(\epsilon,\delta)$-differentially private algorithm under sliding window model that uses $O\left(\frac{dr^2}{\eta^2}\log W\right)$ space and outputs a PSD matrix $S$ at the end of the stream such that 
\vspace{-2mm}
\[
     (1-\eta) A_{W}^\top A_W - \nu\log\nu \I_d \preceq S  \preceq (1+\eta) A_W^\top A_W  + \nu\log\nu \I_d,  \text{ where } \nu =  O\left(\frac{d\log^2 \paren{1/\delta}}{\epsilon^2}\right). 
\]
\end{theorem}

In the static setting, using the result of Sarlos~\citet{sarlos2006improved} and Blocki et al.~\citet{blocki2012johnson}\footnote{The output of Dwork et al.~\citet{dwork2014analyze} does not provide spectral approximation because their output is not  a PSD.}, we can get an $O(d^2)$ space private algorithm which is also an $\paren{\eta,  \frac{d \log (1/\delta) }{\epsilon^2 \eta}}$-spectral approximation algorithm. There is a non-private spectral approximation algorithm in the sliding window model that uses $O\left(\frac{rd}{\eta}\log W\right)$ space under bounded condition number of the matrix~\cite{braverman2018numerical}. In many practical scenarios, $r=O(1)$, in which case {\em we do not pay any price for preserving privacy}. Some example where this is the case is in privacy preserving learning, where the dimension $d$ of the parameter space is large but the gradients of the loss function $\ell$ are empirically observed to be contained in a low-dimensional subspace of $\R^d$~\cite{gur2018gradient, li2020hessian, papyan2019measurements}.

One main application of Theorem~\ref{thm:privslidinginformal} is the principal component analysis and its variants. Principal component analysis is an extensively used subroutine in many applications like clustering~\cite{cohen2015dimensionality,McSherry01},  data mining~\cite{AFKMS01}, recommendation systems~\cite{DKR02}, information retrieval~\cite{PRTV00}, and learning  distributions~\cite{AM05}. 
In these applications, given a matrix $A \in \R^{n \times d}$ and a target rank $k$, the goal is to output a rank-$k$ orthonormal projection matrix $P \in \R^{d \times d}$ such that 
\[
\norm{A - AP}_F \leq (1+\eta) \min_{\mathsf{rank}(X)\leq k} \norm{A - X}_F + \zeta.
\]

The goal here is to minimize $\zeta$ for a given $k,d$, and privacy parameters $\epsilon$ and $\delta$. In many applications, instead of optimizing over all rank-$k$ projection matrices, we are required to optimize over a smaller set of projection matrices, such as one with only non-negative entries. In particular, let $\Pi$ be any set of rank-$k$ projection matrices (not necessarily set of all rank-$k$ projection matrices). Then the {\em restricted principal component analysis} is to find $P^* = \argmin_{P \in \Pi} \norm{A - AP}_F^2.$ 

{
A naive application of approximate spectral histogram property to solve PCA will lead to an additive error that depends linearly on the rank of the streamed matrix. To solve these problems optimally, we introduce an intermediate problem that we call {\em private projection preserving summary} (Definition~\ref{defn:ppcp}).  This problem can be seen as a private analogue of projection-cost preserving sketches~\cite{cohen2015dimensionality}. Solving this problem ensures that the additive error incurred is small. 
}
We consider the first $k/\eta$ spectrum of the streamed matrix and show that it suffices for our purpose. That is, let $\widetilde A_1, \cdots, \widetilde A_\ell$ be matrices such that its covariance matrices $\widetilde S_1, \cdots, \widetilde S_\ell$ satisfy $\eta$-approximate spectral histogram property. We show that random projection of $\widetilde A_1, \cdots, \widetilde A_\ell$ to a $k/\eta$ dimensional linear subspace suffices. Let $\pi_{k/\eta}(\widetilde A_1), \cdots, \pi_{k/\eta}(\widetilde A_\ell)$ be these projected matrices. We show that the set of covariance matrices corresponding to $\pi_{k/\eta}(\widetilde A_1), \cdots, \pi_{k/\eta}(\widetilde A_\ell)$ satisfy the approximate spectral histogram property. Using this, we show that the first matrix, $\widetilde A := \pi_{k/\eta}(\widetilde A_1)$, in this set is a {\em private projection preserving summary} for $A_W$ with a small additive error. \begin{theorem}
[Informal version of Theorem~\ref{thm:constrainedLRA}]
\label{thm:informalconstrainedLRA}
Let $A_W$ be the matrix formed by last $W$ updates and $\Pi$ be a given set of rank-$k$ projection matrices. Then there is an efficient $(\epsilon,\delta)$-differentially private algorithm that outputs a  matrix $P\in \Pi$ at the end of the stream, such that if $\Vert {\widetilde A (\I_d -  P) }\Vert_F \leq \gamma \cdot \min_{X \in \Pi} \Vert {\widetilde A (\I_d -  X)}\Vert_F$ for some $\gamma >0$, then 
\vspace{-1mm}
\[
\norm{A_W(\I_d - P)}_F \leq \paren{{1+\eta}} \gamma \cdot \min_{X \in \Pi} \norm{A_W(\I_d - X)}_F + O\paren{\frac{1}{\epsilon}\sqrt{kd \log(1/\delta)}  }.
\]
\end{theorem}

The matrix $P$ in the above result can be computed using any known non-private algorithm and there are existing results for  structured projection matrices (for example,~\cite{asteris2014nonnegative, d2005direct, papailiopoulos2013sparse, yuan2013truncated, zass2007nonnegative}). In particular, if $\Pi$ is a set of sparse or non-negative projection matrices, then the theorem gives a way to solve these problems privately. We remark that Theorem~\ref{thm:constrainedLRA} also implies a private algorithm for principal component analysis by using algorithms for PCA that achieves $\gamma=1$~\cite{eckart1936approximation}.

\begin{table} [t]
\begin{center}
\small{
\begin{tabular}{|c|c|c|c|c|}
\hline
& Additive Error & Multiplicative?  & Space Required & Comments  \\ \hline
Hardt-Roth~\citet{hardt2012beating}  & $\widetilde O  \left({\epsilon}^{-1}k \| A \|_\infty \sqrt{n}  \right)$  & $\surd$ &  $O(d^2)$ & Static data
\\ \hline 

Dwork et al.~\citet{dwork2014analyze}  & $\widetilde O\left({\epsilon}^{-1} k \sqrt{d} \right)$ & $\times$ & $\widetilde O\left(d^2\right)$ & Static data
\\ \hline

Upadhyay~\citet{upadhyay2018price}    & $\widetilde{O}\left( {\epsilon}^{-1} \sqrt{kd}\right) $ & $\surd$ & $\widetilde O\left(\eta^{-1} dk\right)$ & Streaming data
\\ \hline

Lower Bound   & {$  \Omega \left( \sqrt{kd}\right)   $}~\citet{hardt2012beating} & $\surd$ & \textcolor{red}{$ \Omega\left(\eta^{-1}{dk} \log W\right)$} & \textcolor{red}{Sliding window}
\\ \hline 

This Paper & \textcolor{red}{$\widetilde O\left(\epsilon^{-1} \sqrt{kd}\right)$}& $\surd$ &  \textcolor{red}{$\widetilde O\left(\eta^{-3}{dk^2}\log W\right)$} & \textcolor{red}{Sliding window}  
\\ \hline

\end{tabular}
\caption{Comparison of $\left(\epsilon, \Theta\left(d^{-\log d}\right)\right)$-Differentially private PCA results (our results are in red).} \label{table:basicresults}
} \end{center}
\vspace{-4mm}
\end{table}

Using a reduction to the linear reconstruction attack~\citet{dinur2003revealing}, Hardt and Roth~\citet{hardt2012beating} showed that any private (not necessarily differentially private) algorithm for PCA would incur an additive error of order $\Omega(\sqrt{kd})$. We achieve asymptotically tight accuracy bound even in the sliding window model {up to poly-logarithmic factor}  for $\epsilon=O(1)$ and $\delta=\Theta(d^{-\log d})$.
A comprehensive comparison of our result with previous (non-sliding window) algorithms is presented in Table~\ref{table:basicresults}. From the table, one can see that we improve (or match) previous bounds even though we work in a more restrictive setting.  

We finally remark that we do not violate the lower bound of Dwork et al.~\citet{dwork2014analyze}. Their lower bound holds when there is no multiplicative approximation.We only use $O\paren{\frac{dk^2}{\eta^3}\log W}$ space in the sliding window setting, which is an improvement whenever $k\log W/\eta^3 =o(d)$. Dwork et al.~\cite{dwork2014analyze} also studied PCA in the {\em online learning model}~\cite{hazan2019introduction}, which is incomparable to the sliding window model\footnote{The online learning model is a game between a decision-maker and an adversary. The decision-maker makes decisions iteratively. After committing to a decision, it suffers a loss. These losses can be adversarially chosen, and even depend on the action taken by the decision-maker. The goal is to minimize the total loss in retrospect to the best decision the decision-maker should have taken.}.

\section{Notation and Preliminaries}
\label{sec:notations}
We use the notation $\mathbb{R}$ to denote the space of real numbers and $\mathbb{N}$ to denote the set of natural numbers. For $n \in \mathbb{N}$, we let $[n]$ denote the set $\{1, \dots, n\}$.

\paragraph{Linear algebra.} The space of $n$-dimensional vectors over reals is denoted $\mathbb{R}^n$. The set of non-negative vectors (also known as non-negative orthant) and the set of strictly positive vectors in $\R^n$ are denoted $\R^n_+$ and $\R^n_{++}$, respectively. For a vector $x$, we let $x^{\top}$ denote the transpose of the vector. We reserve the letters $x, y, z$ to denote real vectors. The entries of a vector $x \in \mathbb{R}^n$ is denoted as follows:
\[
x = \big(x[1], x[2], \cdots, x[n] \big)^{\top}.
\]
We let $\left\{ \bar{e}_i : i \in [n]\right\}$ (where $[n] := \{1, 2, \cdots, n\}$) denote the set of standard basis vectors of $\mathbb{R}^n$. That is,
\begin{equation*}
 \bar{e}_i[j] = \left\{
\begin{array}{rl}
1 & \text{if } i = j,\\
0 & \text{if } i \neq j.
\end{array} \right.
\end{equation*}
We let the vector of all 1's denoted by $ \bar{e}$, i.e., $ \bar{e} =  \bar{e}_1 +  \bar{e}_2 + \dots +  \bar{e}_n$. For two vectors $x,y \in \R^n$, their inner product is denoted $\ip{x}{y}$ and is defined as 
\[
\ip{x}{y} := \sum_{i=1}^n x[i]y[i].
\]


\noindent The set of real $n \times m$ matrices is denoted $\mathbb{R}^{n \times m}$. For a  real matrix $A$, the $(i,j)$ entry of it is denoted $A[i,j]$ and its transpose is denoted $A^{\top}$. 
The following special classes of matrices are relevant to this paper.

\begin{mylist}{5mm}
\item[1.]
A real matrix $A \in \mathbb{R}^{n \times n}$ is {\it symmetric} if $A = A^{\top}$. The set of symmetric matrices is denoted $\mathbb{S}^n$ and forms a vector space over $\mathbb{R}$. The eigenvalues of symmetric matrices are real.

\item[2.]
A symmetric matrix $A \in \mathbb{S}^n$ is {\it positive semidefinite} if all of its eigenvalues are non-negative. The set of such matrices is denoted $\mathbb{S}^n_+$. The notation $A\succeq 0$ indicates that $A$ is positive semidefinite and the notations $A\succeq B$ and $B\preceq A$ indicate that $A - B\succeq 0$ for symmetric matrices $A$ and $B$. 
We also use the notation $A \not\succeq B$ and $B \not\preceq A$ for $A, B \in \mathbb{S}^n$ to say that $A - B \not\in \mathbb{S}^n_+$. 

\item[3.]
A positive semidefinite matrix $A \in \mathbb{S}^n_+$ is {\it positive definite} if all of its eigenvalues are strictly positive. The set of such matrices is denoted $\mathbb{S}^n_{++}$. The notation $A\succ 0$ indicates that $A$ is positive definite and the notations $A\succ B$ and $B\prec A$ indicate that $A - B\succ 0$ for symmetric matrices $A$ and $B$.

\item [4.]
A matrix $U \in \R^{n\times n}$ is {\it orthonormal} if $UU^{\top} = U^{\top}U = \mathds{1}_n$, where $\mathds{1}_n$ is the identity matrix. We will drop the subscript $n$ from $\mathds{1}_n$ when the dimension is understood from the context.

\item [5.] A symmetric matrix $P \in \mathbb{S}^n$ is a rank-$k$ {\it orthogonal projection matrix} if it satisfies $P^2 = P$ and it's rank is $k$. Such matrices have eigenvalues $0$ and $1$. A projection matrix that is not orthogonal is called an {\it oblique projection matrix}.
\end{mylist}

The eigenvalues of any symmetric matrix $A \in \mathbb{S}^n$ are denoted by $(\lambda_1(A),\ldots,\lambda_n(A))$ sorted from largest to smallest: $\lambda_1(A) \geq \lambda_2(A) \geq \cdots \geq \lambda_n(A)$. When discussing the largest and smallest eigenvalues, we alternately use the notation $\lambda_{\max}(A)$ and $\lambda_{\min}(A)$ to denote $\lambda_1(A)$ and $\lambda_n(A)$, respectively. Similarly, the singular values of $A$ is denoted by the tuple $(s_1(A),\ldots,s_n(A))$ sorted from largest to smallest:
$s_1(A) \geq s_2(A) \geq \cdots \geq s_n(A)$. We use the notation $s_{\max}(A)$ and $s_{\min}(A)$ to denote the largest and smallest singular values of $A$, respectively. It is a well known fact that for any symmetric matrix $A$
\[
s_{\max}(A) = \max \left\{\vert\lambda_{\max}(A)\vert, \vert\lambda_{\min}(A)\vert \right\} \qquad \text{and} \qquad s_{\min}(A) = \min \left\{\vert\lambda_{\max}(A)\vert, \vert\lambda_{\min}(A)\vert \right\}.
\]
The maximum number of non-zero singular values of $A \in \mathbb{R}^{n \times m}$ is $\min\{n, m\}$. The {\it spectral norm} of a matrix $A \in \mathbb{R}^{n \times m}$ is defined as
\[
\norm{A}_2 = \max\{\norm{Ax}_2\,:\,x\in\mathbb{R}^m,\,\norm{x}_2 = 1\}.
\]
The spectral norm of $A$ is equal to the largest singular value of $A$. The trace norm of a rank-$r$ symmetric matrix $A$ is defined as the sum of its singular values.
\[
\|A\|_1:= \sum_{i=1}^r \vert s_i(A) \vert.
\] 
The Frobenius norm of a matrix $A$ is defined as  	
\[ \|A\|_F := \paren{\sum_{ij} \vert A[i,j] \vert^2}^{1/2}.\] 
It is well known that the Frobenius norm of a rank-$r$ matrix $A$ is 
\[
\norm{A}_F := \paren{\sum_{i=1}^r \vert s_i(A) \vert^2}^{1/2}.
\]
This directly implies that 
$\norm{A}_F^2 = \tr{A^\top A}$ for any $n \times d$ matrix $A$.

Two types of matrix decomposition are used in this paper. The first matrix decomposition is {\it spectral decomposition} (or {\it eigenvalue decomposition}). It means that a symmetric matrix $A \in \mathbb{S}^n$ can be written as
\[
A = U \Lambda U^{\top} = \sum_{i=1}^n \lambda_i(A)x_i x_i^{\top}
\]
where $U$ is an orthonormal matrix, $\Lambda$ is a diagonal matrix with eigenvalues of $A$ on its diagonal, and the set $\left\{x_i \in \mathbb{R}^n : i \in [n]\right\}$ are set of orthonormal vectors known as eigenvectors of $A$. We note that orthonormal matrices can also be decomposed in above form. The second matrix decomposition that is relevant to this paper is {\it singular value decomposition} (or SVD for short). Any real matrix $A \in \mathbb{R}^{n \times d}$ can be decomposed as follows:
\[
A = USV^{\top} = \sum_{i=1}^{\min\{n,d\}}s_i(A)x_iy_i^{\top}.
\]
Here, $U \in \mathbb{R}^{n \times n}$ and $V \in \mathbb{R}^{d \times d}$ are orthonormal matrices, $S$ is a diagonal matrix with  diagonal entries singular values of $A$, and the sets $\left\{x_i \in \mathbb{R}^n : i \in \min\{n,d\}\right\}$ and $\left\{y_j \in \mathbb{R}^d : j \in \min\{n,d\}\right\}$ are orthonormal sets of vectors. Associated with any real matrix $A \in \R^{n \times d}$ is a matrix $A^\dagger$ called the {\it Moore-Penrose pseudoinverse} (or, {\it pseudoinverse}) and is defined as
\[
A^\dagger = \sum_{i=1}^{\min\{n,d\}}\frac{1}{s_i(A)}x_iy_i^{\top} \quad\text{where}\quad
A = \sum_{i=1}^{\min\{n,d\}}s_i(A)x_iy_i^{\top}.
\]
We use the notation
\[
[A]_k := \sum_{i=1}^{\min\{k,n,d\}}s_i(A)x_iy_i^{\top}
\]
to denote the best rank-$k$ approximation of matrix $A$ under any unitary invariant norm.


We consider matrices formed by a streams of $d$-dimensional row vectors over $\R$. For a row vector $a_t \in \R^d$ streamed at time $t$, we use the notation $\bar A(a_t) \in \R^{W \times d}$ to denote an all zero matrix except row $t$ which is $a_t$ if $t\leq W$ and row $W$ to be $a_t$ if $t > W$. For time epochs, $t_1$ and $t_2$, we define the matrix $A_{[t_1,t_2]} \in \R^{(t_2-t_1) \times W}$  
to denote the matrix formed by stacking the row vectors $a_{t_1}, \cdots, a_{t_2} \in \R^d$ streamed from time epoch $t_1$ to $t_2$:
\[
    A_{[t_1,t_2]} :=
    \begin{pmatrix}
        a_{t_1} \\ a_{t_1+1} \\ \vdots \\ a_{t_2} 
    \end{pmatrix}
\]

\paragraph{Probability distributions.} For a random variable $X \in \mathbb{R}$, we denote the mean and variance of $X$ by $\mathbb{E}[X]$ and $\mathsf{Var}(X)$, respectively. The symbols $\mu$ and $\sigma^2$ are reserved to represent these quantities. 
We say that a random variable $X \in \mathbb{R}$ has Gaussian (or normal) distribution if its probability density function is given by
\[
p(x) = \frac{1}{\sqrt{2\pi}\sigma} \exp\left(-\frac{(x - \mu)^2}{2\sigma^2}\right).
\]
We denote Gaussian distribution by $\cN\left(\mu, \sigma^2\right)$ and write $X \sim \cN\left(\mu, \sigma^2\right)$ when $X$ has Gaussian distribution.

 Throughout this paper, we discuss and work with {\it random matrices}. They are simply matrices with matrix entries drawn from random variables that may or may not be independent. 
A special class of random matrices are Wishart matrices. They are defined as below. Let $C$ is a $d \times d$ positive definite matrix and $m >d-1$. A $d \times d$ random
symmetric positive definite matrix $R$ is said to have a Wishart distribution $R \sim \mathsf{Wis}_d(m, C)$, if its probability density function is
\[
p(W) := \frac{|R|^{\frac{m-d-1}{2}}}{2^{md}|C|^{\frac{m}{2}} \Gamma_d(\frac{m}{2})} \exp \paren{ \frac{1}{2} \tr{C^{-1}R} },
\]

We also consider the {\em constrained principal component analysis}~\cite{cohen2015dimensionality}. This is a generalization of many variants of principal component analysis such as sparse and non-negative principal component analysis.
\begin{definition}
[Constrained principal component analysis~\cite{cohen2015dimensionality}]
\label{defn:restrictedpca}
\label{defn:pca}
Given a matrix $A_W$ formed by a window of size $W$, a rank parameter $k$, a set of rank-$k$ projection matrices $\Pi$ and an accuracy parameter $0< \eta <1$, design a differentially private algorithm under sliding window model which outputs a projection matrix $P \in \Pi$ such that
\[
\norm{A - PA}_F \leq (1 + \eta) \min_{X \in \Pi} \norm{A - XP}_F + \tau
\]
with probability at least $1-\beta$.  Furthermore, the algorithm should satisfies $(\epsilon,\delta)$-differential privacy. When $\Pi$ is a set of all rank-$k$ projection matrices, then we get the traditional problem of principal component analysis.  
\end{definition}

\section{Approximate spectral histogram} 
\label{sec:slidingJL}

In this section, we introduce the central contribution of this paper:  {\em $\eta$-approximate spectral histogram property}.  We later show in Section~\ref{sec:applicationJL} that these properties are sufficient to perform all the matrix analysis mentioned in Section~\ref{sec:introduction} efficiently with respect to time, efficiency, and accuracy. In particular, by maintaining this property, we can maintain an intrinsic rank dependent bound. To maintain this property, our data structure stores a set of random matrices that approximates the spectrum of the original matrices.  
 
Let $\widetilde A$ denotes the $\eta$-spectral approximation of the matrix $A$.  Then $\eta$-approximate spectral histogram property is formally defined as follows:

\begin{definition}
[$\eta$-approximate-spectral histogram property] 
\label{defn:approxsmoothPSD}
A data structure $\mathfrak{D}$ satisfy $\eta$-{\em approximate spectral histogram} property if there exists an $s=\poly(n,\log W)$ such that $\mathfrak{D}$ satisfy the following conditions
\end{definition}
\begin{enumerate}
\item
$\mathfrak{D}$ consists of $\ell $ timestamps  $\mathsf I:=\{ t_{1} , \cdots, t_{\ell} \} $ and the corresponding matrices $\mathsf{S}:=\{ \widetilde S_1,  \cdots, \widetilde S_\ell \}$. 
\item
For $1\leq i\leq s-1$, at least one of the following holds:
\label{item:approxsmoothPSD}
\begin{itemize}
\item
If $t_{i+1} = t_{i}+1$, then $\paren{1-\frac{\eta}{2}} \widetilde S_i \not\preceq \widetilde S_{i+1}.$ \label{item:technical1}
\item
\label{item:combined1}
For all $1\leq i \leq \ell-2$:
\begin{enumerate}
\item
$(1-\eta) S_i \preceq  S_{i+1}$. \label{item:approx_prop1}
\item
$\paren{1-\frac{\eta}{2}} \widetilde S_i \not\preceq \widetilde S_{i+2}$, where $\widetilde S_i$ is $\frac{\eta}{4}$-spectral approximation of $S_i$ and $\widetilde S_{i+2}$ is $\frac{\eta}{4}$-spectral approximation of $S_{i+2}$. \label{item:approx_prop2}
\end{enumerate}
\end{itemize}
\item Let $A_W$ be the matrix formed by the window $W$. Then $S_2 \preceq A_W^\top A_W \preceq S_1.$
\label{item:approxPSD}
\end{enumerate}

The difference between spectral histogram property~\cite{braverman2018numerical} and $\eta$-approximate spectral histogram property is that in $\eta$-approximate spectral histogram property, we allow the matrices in the data structure to be a spectral approximation of the corresponding original matrices and in properties defined in item~\ref{item:approxsmoothPSD}. This relatively simple relaxation is crucial in achieving optimal results for our downstream matrix analysis while preserving differential privacy. 

We next show that we can maintain a data structure that allows efficient updates and a sequence of matrices that satisfy the $\eta$-approximate spectral histogram property using an algorithm {\scshape Update-Approx}.  The case of {\em continual observation}~\cite{dwork2010differentially} is covered in more detail in the appendix.  We will later see (in equation~(\ref{eq:relationtildehat})) that the input to the {\scshape Update-Approx} is constructed in a specific manner; therefore,  the matrices in the sequence have a particular form and satisfy the Loewner ordering. 

\begin{algorithm}[htbp]
\caption{{\scshape Update-Approx}$(\dpriv)$}
\label{fig:slidingprivpcpupdate}
\begin{algorithmic}[1]
\Require{A data structure $\dpriv$ storing  a set of matrices ${\widetilde A_{[t_1,t]}}, \cdots,  {\widetilde A_{[t_\ell,t]}}$ such that 
$$\widetilde A_{[t_i,t]}^\top \widetilde A_{[t_i,t]} \succeq \widetilde A_{[t_{i+1},t]}^\top \widetilde A_{[t_{i+1},t]} \text{ for all } 1 \leq i \leq \ell-1$$
and corresponding time stamps  $t_{1}, \cdots, t_{\ell}$.} 
\Ensure{Updated matrices $ {\widetilde A_{[t_1,t]}}, \cdots,  {\widetilde A_{[t_\ell,t]}}$ and timestamps $t_{1}, \cdots, t_{\ell}$.}

\State {\bf Define} for $1 \leq i \leq \ell$,
    $$\widetilde K(i) := \widetilde A_{[t_i,t]}^\top \widetilde A_{[t_i,t]}.$$

\State {\bf For }{$i=1, \cdots \ell-2$}
\label{step:stage2startpcp}
\label{step:smoothupdatingmainpcp}
    \State \hspace{5mm}  Find    \Comment{\small{Find spectrally close checkpoints.}}
    \begin{align}
            j := \max \set{ p : \paren{1-\frac{\eta}{2}} \widetilde K(i) \preceq \widetilde K(p) \wedge (i < p \leq \ell-1) }
    \label{step:checkmainpcp}
    \end{align}

    \State \hspace{5mm} {\bf Delete} ${\widetilde A_{[t_{i+1},t]}}, \cdots, {\widetilde A_{[t_{j-1},t]}}$. 
    \label{step:deletemainpcp}
    
    \State \hspace{5mm} {\bf Set} k=1
    \State \hspace{5mm} \textbf{While} {$i+k \leq \ell $}
    \label{step:updatemainpcp}
         \State \hspace{1cm} {\bf Update} the checkpoints: ${\widetilde A_{[t_{i+k},t]}}= {\widetilde A_{[t_{j+k-1},t]}}$, $t_{i+k} = t_{j+k-1}$.
    \State \hspace{5mm} {\bf end}
    \State \hspace{5mm} $s:=s+i-j+1$.
\label{step:stage2endpcp}

\State {\bf Define} $\dpriv := \set{(\widetilde A_{[t_1,t]}, t_1), \cdots ,(\widetilde A_{[t_\ell,t]}, t_\ell)}$.

\State {\bf Output} $\dpriv.$
\end{algorithmic}
\end{algorithm}

\begin{lemma}
\label{lem:smoothLaplacianproperty}
Let $\mathfrak{D}$ be a data structure that at time $T$ consists of $\ell$ tuples $\left\{(t_i,S_i)\right\}_{i=1}^{\ell}$, where $0< t_i \leq T$ for all $i$ and 
$$S_\ell^\top S_\ell \preceq \ldots \preceq S_1^\top S_1^\top.$$ 
Let $\mathfrak{D}_{\textsf{spectral}}\leftarrow${\scshape Update-Approx}$(\mathfrak{D})$ be the output of the algorithm {\scshape Update-Approx}, defined in Algorithm~\ref{fig:slidingprivpcpupdate}. 
Then if $\mathfrak{D}_{\textsf{spectral}} := \{ (\bar t_i, B_{1}), \ldots ,  (\bar t_m,B_{m}) \}$ for some $m \leq \ell+1$, then $\{ (\bar t_1, B_{1}^\top  B_{1}), \ldots ,  (\bar t_m,B_{m}^\top  B_{m}) \}$
 satisfy the $\eta$-approximate spectral histogram property. Moreover, the algorithm $\mathfrak{D}_{\textsf{spectral}}$ runs in $\poly(d, \ell)$ time. 
\end{lemma}
\begin{proof}
The run-time of $\mathfrak{D}_{\textsf{spectral}}$ is straightforward, so we only concentrate on the correctness part. Consider a time epoch $t$ and a succeeding time epoch $t'=t+1$. First notice that if we cannot find a $j$ in equation~(\ref{step:checkmainpcp}) for all $i =1, \cdots, \ell-2$, then $\ell = O\left(\frac{d\log W}{\eta}\right)$. This is because the  data structure $\dpriv$ satisfied $\eta$-approximate spectral histogram property at time $t$ and if no such $j$ exists, then none of the properties are violated due to an update. As a result, equation~(\ref{eq:boundell}) gives us that  $\ell = O\left(\frac{d\log W}{\eta}\right)$. Hence, we assume that this is not the case and the data structure got updated between time epochs $t$ and $t+1$. 

Let the data structure at time $t$ be  $\dpriv(t)$ and at time $t'$ be $\dpriv(t')$. That is,
\begin{align*}
    \dpriv(t) &:= \set{(t_1,\widetilde A_{[t_1,t]}); \cdots, (t_\ell, \widetilde A_{[t_\ell,t]})} 
    \\ 
    \dpriv(t') &:= \set{(t_1',\widetilde B_{[t_1',t+1]}); \cdots, (t_\ell', \widetilde B_{[t_\ell',t+1]})}.
\end{align*}
Let $t_i$ be a timestamp in $\dpriv(T)$ where $i <\ell $. We can have two possibilities: 
\begin{enumerate}
    \item There is no $c \in [s]$ such that $t_c'=t_i$ and $t_{c+1}' = t_{i+1}$.
    \item There is a $c \in [s]$ such that $t_c'=t_i$ and $t_{c+1}' = t_{i+1}$.
\end{enumerate}

This is where we deviate from Braverman et al.~\cite{braverman2018numerical} in two respects: first, we need to deal with the approximation factor and perturbation due to privacy, second, we simplify the second case which requires two subcases in the proof of Braverman et al.~\cite{braverman2018numerical}. Fix the following notations for all $1 \leq j \leq \ell$: 
\begin{align*}
    \widetilde K(j) &:= \widetilde A_{[t_j,t]}^\top \widetilde A_{[t_j,t]} 
     \qquad \text{and} \qquad
    \widetilde L(j) :=\widetilde B_{[t_j,t+1]}^\top \widetilde B_{[t_j,t+1]}.
\end{align*}

Let $K(j)$ and $L(j)$ be such that 

\begin{align}
    \paren{1 - \frac{\eta}{4}} K(j)  &\preceq \widetilde K(j)  \preceq \paren{1 + \frac{\eta}{4}} K(j)   \qquad \text{and} \qquad
    \paren{1 - \frac{\eta}{4}} L(j) \preceq \widetilde L(j)  \preceq \paren{1 + \frac{\eta}{4}}  L(j).
    \label{eq:spectralapprox}
\end{align}

As we will see later (in equation~(\ref{eq:relationtildehat})), the input to the {\scshape Update-Approx} are constructed in a specific manner that will ensure that  $K(j)$ and $L(j)$ satisfy the Loewner ordering. Let us consider the first possibility, i.e., there is no $j \in [s]$ such that $t_j'=t_i$ and $t_{j+1}' = t_{i+1}$. By the update rule, we have 
\[
\paren{1-\frac{\eta}{2}} \widetilde  L(j)  \preceq \widetilde  L(j+1) .
\]

Using Proposition~\ref{prop:inequalityeta}, we have
\begin{align*}
    (1-\eta)  L(j) &\prec \paren{1-\frac{\eta}{2}} \paren{\frac{1-\frac{\eta}{4}}{1+\frac{\eta}{4}}}  L(j).
\intertext{Equation~(\ref{eq:spectralapprox}) gives us a relationship between $L(j)$ and $\widetilde L (j)$: }
    \paren{1-\frac{\eta}{2}} \paren{\frac{1-\frac{\eta}{4}}{1+\frac{\eta}{4}}}  L(j) &\preceq \paren{\frac{1-\frac{\eta}{2}}{1+\frac{\eta}{4}}} \widetilde  L(j).  \\
\intertext{From the update rule, it follows that}
    \paren{\frac{1-\frac{\eta}{2}}{1+\frac{\eta}{4}}} \widetilde  L(j) & \preceq \frac{1}{1+\frac{\eta}{4}} \widetilde  L(j+1)
\intertext{Finally, another application of Equation~(\ref{eq:spectralapprox}) gives us}
    \frac{1}{1+\frac{\eta}{4}} \widetilde  L(j+1) & \preceq  L(j+1). 
\end{align*}
Therefore, we have the relation listed in item~\ref{item:approx_prop1} in Definition~\ref{defn:approxsmoothPSD}. Furthermore, due to the maximality condition in the update rule, there exists an $k \in [d]$ such that 
\begin{align*}
    s_k(\widetilde  L(j+2)) < \paren{1 - \frac{\eta}{2}} s_k(\widetilde  L(j)),
\end{align*}
where $\set{s_1(D), \cdots, s_d(D)}$ denotes the singular values of an $d \times d$ matrix $D$. This proves the statement of the Lemma~\ref{lem:smoothLaplacianproperty} for the first scenario. 

Now let us consider the second possibility. Again by the update rule, we have 
\[
\paren{1-\frac{\eta}{2}} \widetilde  L(j)  \preceq \widetilde  L(j+1) .
\]

\noindent It follows from Proposition~\ref{prop:inequalityeta} and Theorem~\ref{thm:sketch} that
\begin{align*}
    (1-\eta)  L(j) 
    \prec \paren{1-\frac{\eta}{2}} \paren{\frac{1-\frac{\eta}{4}}{1+\frac{\eta}{4}}} L(j) 
    \preceq \paren{\frac{1-\frac{\eta}{2}}{1+\frac{\eta}{4}}} \widetilde L(j) .
\end{align*}

\noindent Another application of the update rule and Theorem~\ref{thm:sketch} gives us
\begin{align*} 
    \paren{\frac{1-\frac{\eta}{2}}{1+\frac{\eta}{4}}} \widetilde L(j) 
    &\preceq \frac{1}{1+\frac{\eta}{4}} \widetilde L(j+1)  
    \preceq L(j+1), 
\end{align*}
where the last positive semidefinite inequality follows from Theorem~\ref{thm:sketch}. The second part of {\em $\eta$-approximate spectral histogram property} follows similarly as in the case of the first case. This proves the statement of the Lemma~\ref{lem:smoothLaplacianproperty} for the second sdenario. 
\end{proof}

Lemma~\ref{lem:smoothLaplacianproperty} gives the guarantee that, if we are given a set of positive semidefinite matrices in the Loewner ordering, then we can efficiently maintain a small set of positive semidefinite matrices that satisfy $\eta$-approximate spectral histogram property. The idea of our algorithm for spectral approximation is to ensure that {\scshape Update-Approx} always receives a set of positive semidefinite matrices. This is attained by our algorithm {\scshape Priv-Initialize-Approx}, which gets as input a new row and updates all the matrices in the current data structure. We use both these subroutines in our main algorithm, {\scshape Sliding-Priv-Approx}, that receives a stream of rows and call these two subroutines on every new update. Equipped with Lemma~\ref{lem:smoothLaplacianproperty}, we show that  {\scshape Sliding-Priv-Approx}, described in Algorithm~\ref{fig:slidingpcpprivgraph} provides the following guarantee.

\begin{algorithm}[htbp]
\caption{{\scshape Sliding-Priv-Approx}$(\Omega;r;(\epsilon,\delta);W)$}
\label{fig:slidingpcpprivgraph}
\begin{algorithmic}[1]
\Require{A stream, $\Omega$, of row vectors $\set{a_t}$ and the desired rank of matrices, $r$, privacy parameters $(\epsilon,\delta)$, and window size $W$.} 
\Ensure{A matrix $ {\widetilde A} \in \R^{\frac{4r}{\eta} \times d}$ at the end of the stream.}

\State {\bf Initialize} $\dpriv$ to be an empty set and $\Phi \in \R^{\frac{4r}{\eta} \times (d+1)}$ such that every entry $\Phi[i,j] \sim \cN(0,\frac{\eta}{4r})$.

\State {\bf while} {stream $S$ has not ended}
    \State \hspace{5mm}{\bf Include row:} $\dpriv \leftarrow$ {\scshape Priv-Initialize-Approx} $(\dpriv;a_t; t; (\epsilon,\delta); W; \Phi; r)$. \\ \Comment{Algorithm~\ref{fig:slidingprivpcpinitial}}
    
    \State \hspace{4mm} {\bf Update} the data structure $\dpriv \leftarrow${\scshape Update-Approx}$(\dpriv)$.
    \Comment{Algorithm~\ref{fig:slidingprivpcpupdate}}
    
\State {\bf end}

\State {\bf Let} $\dpriv = \set{(t_i,\widetilde A_{[t_i,t]})}_{i=1}^\ell $ for some $\ell $.

\State {\bf Output} $\widetilde A = \widetilde A_{[t_1,t]}$.
\end{algorithmic}
\end{algorithm}
\begin{theorem}
[Private spectral approximation under sliding window]
\label{thm:privslidinglra}
Given the privacy parameter $\epsilon,\delta$, window size $W$, desired rank $r$ approximation parameter $\eta$, and $S = (a_t)_{t>0}$ be the stream such that $a_t \in \R^d$. Let $A_W$ be the matrix formed at time $T$ using the last $W$ updates as defined in equation~(\ref{eq:sliding}). 
Then we have the following:
\begin{enumerate}
    \item {\scshape Sliding-Priv-Approx}$(\Omega;r;(\epsilon,\delta);W)$ is $(\epsilon,\delta)$-differential privacy.
    \item $ \widetilde A\leftarrow ${\scshape Sliding-Priv-Approx}$(\Omega;r;(\epsilon,\delta);W)$ satisfies the following with probability $9/10$:
    \[
    { \paren{1-\frac{\eta}{4}} \paren{A_W^\top A_W + \frac{c r \log^2(1/\delta)}{\epsilon^2} \I_d } \preceq \widetilde A^\top \widetilde A \preceq \frac{(1 + \frac{\eta}{4})^2}{(1-\eta)} A_W^\top A_W  + \frac{c r \log^2(1/\delta)}{\epsilon^2} \I_d.} 
    \]
    \item The space required by {\scshape Sliding-Priv-Approx}$(\Omega;r;(\epsilon,\delta);W)$ is $O(\frac{dr^2}{\eta^2} \log W)$.
    
    \item The update time of {\scshape Sliding-Priv-Approx} is $\poly(d, \log W)$. 
\end{enumerate}
\end{theorem}

\begin{algorithm}[t]
\caption{{\scshape Priv-Initialize-Approx}$(\dpriv; a_t; t; (\epsilon,\delta);  W; \Phi; r)$}
\label{fig:slidingprivpcpinitial}
\begin{algorithmic}[1]
\Require{A new row $a_t \in \mathbb{R}^d$, a data structure $\dpriv$ storing a set of timestamps $t_{1}, \cdots, t_{\ell}$ and set of matrices $${\widetilde A_{[t_1,t]}}, \cdots,  {\widetilde A_{[t_\ell,t]}},$$ current time $t$, window size $W$, and random matrix $\Phi \in \R^{\frac{4r}{\eta} \times d}$.} 

\Ensure{Updated matrices $ {\widetilde A_{[t_1,t]}}, \cdots,  {\widetilde A_{[t_\ell,t]}}$ and timestamps $t_{1}, \cdots, t_{\ell}$.}

\State  {\bf if} {$t_{2} < t- W+1$}
\label{step:stage1startpcp}
    \State  \hspace{5mm} {\bf Set} $t_{j} = t_{j+1},\widetilde A_{[t_j,t]} := \widetilde A_{[t_{j+1},t]} $ for $1\leq j \leq \ell-1$ 
    \Comment{\small{Delete the expired timestamp.}}
    \label{step:expiredmainpcp}
\State {\bf end}
\label{step:t_2expiredmainpcp}


\State {\bf Sample} a gaussian row vector $g \in \R^{\frac{4r}{\eta}}$.

\State {\bf Set} $t_{\ell+1} = t, \sigma = \frac{16\sqrt{r \log(4/\delta)} + \log(4/\delta)}{\epsilon}$, and 
\begin{align}
{\widetilde A_{[t_{\ell+1},t]}} := \begin{pmatrix} \Phi & g^\top \end{pmatrix} \begin{pmatrix} \sigma \I_d \\ a_t \end{pmatrix} \in \R^{\frac{4r}{\eta} \times d}, \quad \widetilde A:= g^\top  a_t   \in \R^{\frac{4r}{\eta} \times d}.
\label{step:newcheckpointpcp}
\end{align}

\State {\bf Update} $\dpriv \leftarrow \dpriv \cup ( {\widetilde A_{[t_{\ell+1},t]}},t)$. 


\State {\bf for} {$i= 2,\ldots, \ell $}
\label{step:updatecheckpointpcp}
\State  \hspace{5mm} {\bf Compute} $ {\widetilde A_{[t_i,t]}}\leftarrow {\widetilde A_{[t_i,t]}} + \widetilde A.$ 
\Comment{\small{Update the matrices.}}
\State {\bf end}
\label{step:stage1endpcp}

\State {\bf Return} $\dpriv := \set{(\widetilde A_{[t_i,t]}, t_i)}_{i=1}^\ell $.
\end{algorithmic}
\end{algorithm}

\begin{proof}
The proof of update time is pretty straightforward.We divide the proof of  Theorem~\ref{thm:privslidinglra} in three parts:  (i) proof of accuracy guarantee (Lemma~\ref{lem:accuracypcp}); (ii) the privacy proof (Lemma~\ref{lem:privacypcp}); and (iii) proof of the space complexity (Lemma~\ref{lem:space}).

\begin{lemma}
[Accuracy guarantee]
\label{lem:accuracypcp}
Let $S:=\set{a_t}_{t=1}^T$ be the streams of row and $A_W$ be the rank $r$ matrix formed in the time epoch $[T-W+1,W]$ as defined in equation~(\ref{eq:sliding}).  Then $\widetilde A \leftarrow$ {\scshape Sliding-Priv-Approx}$(\Omega;r;(\epsilon,\delta);W)$ satisfies the following
\[
  \paren{1-\frac{\eta}{4}} \paren{A_W^\top A_W + \frac{c r \log^2(1/\delta)}{\epsilon^2} \I_d } \preceq \widetilde A^\top \widetilde A
\preceq 
\frac{(1 + \frac{\eta}{4})}{(1-\eta)} A_W^\top A_W  + \frac{c r \log^2(1/\delta)}{\epsilon^2} \I_d
\]
for a constant $c>0$ with probability at least $1-\beta$, where the probability is taken over the internal coin tosses of {\scshape Priv-Initialize-Approx}.
\end{lemma}
\begin{proof}
Note that the output of {\scshape Sliding-Priv-Approx} $(\Omega;r;(\epsilon,\delta);W)$ is $\widetilde A_{[t_1,t]}$, the first matrix in the data structure $\dpriv$. Let us write 
\[
\widehat A_{[t_j,t]}:= \begin{pmatrix} \sigma \I_d \\ A_{[t_j,t]} \end{pmatrix}.
\]

\noindent Fix the following notations: 
\begin{align}
\begin{split}
    K(j) &:=  A_{[t_j,t]}^\top A_{[t_j,t]} \\
    \widehat K(j) &:= \begin{pmatrix} \sigma \I_d \\ A_{[t_j,t]} \end{pmatrix}^\top \begin{pmatrix} \sigma \I_d \\ A_{[t_j,t]} \end{pmatrix} = A_{[t_j,t]}^\top A_{[t_j,t]} + \sigma^2 \I_d = K(j) + \sigma^2 \I_d, \\
    \widetilde K(j) &:= \widetilde A_{[t_j,t]}^\top \widetilde A_{[t_j,t]}  \quad \text{for all} \quad j \in [\ell].
\label{eq:relationtildehat}    
\end{split}
\end{align}
where
\[\sigma := \frac{16\sqrt{r \log(4/\delta)}+ \log(4/\delta)}{\epsilon}.\]

Since the starting time of the window is sandwiched between the first and second timestamps, the matrix $A_W^\top A_W$ is approximated in the following manner:
\begin{align}
    K(2) \preceq A_W^\top A_W \preceq K(1) 
    \label{eq:windowapproxpcp}
\end{align} 

\noindent Moreover, from the $\eta$-approximate spectral histogram property, we have the following relation between $\widehat K(1)$ and $\widehat K(2)$ (equivalently $\widehat A_{[t_1,t]}$ and $\widehat A_{[t_2,t]}$): 
\begin{align}
    (1-\eta)  \widehat K(1) \preceq  \widehat K(2).
    \label{eq:spectralpcp}
\end{align}

\noindent Since $\widetilde K(j)$ is $\frac{\eta}{4}$-spectral approximation of $\widehat K(j)$ for all $j \in [\ell]$, setting $\zeta = \frac{\eta}{4}$ and $A:=\widehat A_{[t_i,t]}$ for $i =\set{1,2}$ in Theorem~\ref{thm:sketch}, we have with probability $1 - {\beta}$,
\begin{align}
\begin{split}
  \paren{1 - \frac{\eta}{4}} \widehat K(1) \preceq \widetilde K(1) \preceq \paren{1 + \frac{\eta}{4}} \widehat K(1). \\
 \paren{1 - \frac{\eta}{4}} \widehat K(2) \preceq \widetilde K(2) \preceq \paren{1 + \frac{\eta}{4}} \widehat K(2).
 \label{eq:approxG_2pcp}
\end{split}
\end{align}
 
\noindent Using equation~(\ref{eq:relationtildehat}) and~(\ref{eq:approxG_2pcp}), we have that
\begin{align}
    \paren{1 - \frac{\eta}{4}} (K(1) + \sigma^2 \I_d) \preceq \widetilde K(1).
    \label{eq:k1tildek1}
\end{align}

Since adding positive semidefinite matrices preserves the Loewner ordering, using equation~(\ref{eq:windowapproxpcp}) we can dedure the following implication from equation~(\ref{eq:k1tildek1}), we get the following:
\begin{align}
    \paren{1 - \frac{\eta}{4}} (A_W^\top A_W + \sigma^2 \I_d) \preceq \paren{1 - \frac{\eta}{4}} (K(1) + \sigma^2 \I_d) \preceq \widetilde K(1).
    \label{eq:pcpleftside}
\end{align}
This completes the proof of the lower bound on $\widetilde K(1)$. To upper bound $\widetilde K(1)$, equation~(\ref{eq:approxG_2pcp}) gives us
\begin{align}
    \widetilde K(1) \preceq \paren{1 + \frac{\eta}{4}} \widehat K(1).
    \label{eq:k1hatk1}
\end{align}

\noindent Combined with equation~(\ref{eq:spectralpcp}), equation~(\ref{eq:k1hatk1}) gives us 
\begin{align}
    \widetilde K(1) \preceq \frac{\paren{1 + \frac{\eta}{4}}}{(1-\eta)} \widehat K(2).
    \label{eq:k1hatk2}
\end{align}

\noindent Combining equation~(\ref{eq:k1hatk2}) with equation~(\ref{eq:relationtildehat}) and  equation~(\ref{eq:windowapproxpcp}) then gives us
\begin{align}
    \widetilde K(1) \preceq \frac{\paren{1 + \frac{\eta}{4}}}{(1-\eta)} \paren{ A_W^\top A_W + \sigma^2 \I_d }.
    \label{eq:pcprightside}
\end{align}
This completes the proof of the upper bound on $\widetilde K(1)$. The statement of the lemma follows by combining equations~(\ref{eq:pcpleftside}) and~(\ref{eq:pcprightside}) and substituting the value of $\sigma$.
 \end{proof} 

\begin{lemma}
[Privacy guarantee]
\label{lem:privacypcp}
Let $S:=\set{a_t}_{t=1}^T$ be the streams of row and $A_W$ be the rank $r$ matrix formed in the time epoch $[T-W+1,W]$ as defined in equation~(\ref{eq:sliding}).  Then $\widetilde A \leftarrow$ {\scshape Sliding-Priv-Approx}$(\Omega;r;(\epsilon,\delta);W)$ is $(\epsilon,\delta)$-differentially private.
\end{lemma}
\begin{proof}
Let $\mathcal P$ denote the output distribution of our mechanism when run on the input matrix $A_{[t_i,t]}$ and similarly let $\mathcal Q$ denote the output of our algorithm on input matrix $A'_{[t_i,t]}$. Both distribution are supported on $\mathbf S:=\R^{d \times d}$ matrices. For $M \in \mathbf S$, consider the privacy loss function 
\[
L(M) := \log \paren{ \frac{\mathcal P(M)}{\mathcal Q(M)} }.
\]
Our goal is to show that 
\begin{align*}
\pr_{M} [L(M) \leq \epsilon] \geq 1-\delta. 
\end{align*}
for all $T$. We accomplish this by considering following three different cases

\begin{itemize}
    \item [1.] When $T < t$: In this case, the output  distributions of $\mathcal P$ and $\mathcal Q$ are identical. It follows that $L(M)=0$.
    
    \item [2.] When $T = t$: In this case, we note that all singular values of matrix is at least $\sigma$. Therefore,  the privacy proof follows from    Theorem~\ref{thm:JLmechanism}, which in particular  implies that 
\[\pr_{M} [L(M) \leq \epsilon] \geq 1-\delta.\]

\item [3.] When $T > t$: This final case follows by invoking the post processing property of differential privacy mechanisms (Lemma~\ref{lem:post}).
\end{itemize}
This completes the proof of privacy guarantee.
\end{proof}

\begin{lemma}
[Space complexity]
\label{lem:space}
Let $\eta \in (0,1/4)$ be the approximation parameter, $(\epsilon,\delta)$ be the privacy parameters, and $W$ be the size of the window. The total space required to maintain the data structure $\dpriv$, and hence the algorithm {\scshape Sliding-Priv-Approx}$(\Omega;r;(\epsilon,\delta);W)$, is $O\left(\frac{r^2d}{\eta} \log n \log W\right)$.
\end{lemma}
\begin{proof}
The number of checkpoints stored by $\dpriv$ is $O\left(\frac{r}{\eta}\log W\right)$. This follows from basic counting argument. There are at most $r$ singular values. Since each checkpoints are at least $(1-\frac{\eta}{2})$ apart for at least one singular value and all the non-zero singular values are bounded by a polynomial, there can be at most
\begin{align}
c \log_{(1-\frac{\eta}{2})} (W) = c \frac{\log W}{\log (1-\frac{\eta}{2})} \leq \frac{2c}{\eta} \log W
\label{eq:boundell}
\end{align}
checkpoints that sees the jump in a specific singular value. 
Since each checkpoint defined by $t_i$ (for $i \geq 2$) stores a covariance matrix, using Theorem~\ref{thm:sketch}, the total space used by the checkpoints $\set{\widetilde A_{[t_i,t]},t_i}_{i \geq 2}$ is $O\left(\frac{rd}{\eta} \log W\right)$. This finishes the proof of Lemma~\ref{lem:space}.
\end{proof}

\noindent Theorem~\ref{thm:privslidinglra} follows by combining Lemma~\ref{lem:privacypcp}, Lemma~\ref{lem:accuracypcp}, and Lemma~\ref{lem:space}.
\end{proof}

Since 
\[
\frac{(1 + \frac{\eta}{4})}{(1-\eta)} \leq 1 + 2 \eta
\]
for $\eta \leq \frac{3}{8}$, a direct corollary of Theorem~\ref{thm:privslidinglra} is the following:

\begin{corollary}
\label{cor:approximation}
Let $\eta \in (0,\frac{3}{8})$. Then {\scshape Sliding-Priv-Approx}$(\Omega;r;(\epsilon,\delta);W)$ outputs a matrix $\widetilde A$ such that 
\[
  (1-\eta) \paren{A_W^\top A_W - \frac{c r \log(1/\delta)}{\epsilon^2} \I_d } \preceq \widetilde A^\top \widetilde A  \preceq  (1 + \eta) \paren{A_W^\top A_W  + \frac{c r \log(1/\delta)}{\epsilon} \I_d}
\]
for some constant $c>0$.
\end{corollary}

If we wish to output a positive semidefinite matrix, one can achieve a better accuracy guarantee than what Theorem~\ref{thm:privslidinglra} gives. The idea is to compute $\widetilde A^\top \widetilde A - \sigma^2 \I_d$. This is encapsulated in the following theorem.
\begin{theorem}
\label{thm:slidingprivapprox}
Given the privacy parameters $(\epsilon,\delta)$, window size $W$, desired rank $r$, and approximation parameter $\beta$, let $S = \{a_t\}_{t>0}$ be the streams such that $a_t \in \R^d$. Let $A_W$ be the matrix formed at time $T$ using the last $W$ updates as defined in equation~(\ref{eq:sliding}). Let $\widetilde A \leftarrow$ {\scshape Sliding-Priv-Approx}$(\Omega;r;(\epsilon,\delta);W)$. Then we can compute a positive semidefinite matrix $C$ such that 
    \[
    \pr \sparen{ \paren{1- \eta} A_W^\top A_W - \frac{c \eta r \log (1/\delta)}{4\epsilon^2} \I_d  \preceq C \preceq \paren{1 + \eta} A_W^\top A_W  + \frac{c \eta r \log(1/\delta)}{\epsilon^2} \I_d.} \geq 1 -\beta.
    \]
    Moreover, computing $C$ preserves $(\epsilon,\delta)$-differential privacy.
\end{theorem}
\begin{proof}
    Compute $\widetilde A \leftarrow$ {\scshape Sliding-Priv-Approx}$(\Omega;r;(\epsilon,\delta);W)$ and let $C = \widetilde A^\top \widetilde A - \sigma^2 \I_d$. Then using Corollary~\ref{cor:approximation}, we have for the lower bound:
    \begin{align*}
     C &= \widetilde A^\top \widetilde A -  \frac{cr \log(1/\delta)}{\epsilon^2} \I_d  \succeq (1-\eta)\paren{A_W^\top A_W - \frac{c r \log(1/\delta)}{\epsilon^2} \I_d } - \frac{cr \log(1/\delta)}{\epsilon^2} \I_d \\
     &= (1-\eta) A_W^\top A_W - \frac{c\eta r \log(1/\delta)}{\epsilon^2} \I_d
    \end{align*}
    as required. The upper bound follows similarly.
\end{proof}


\section{Applications} 
\label{sec:applicationJL}
Algorithm~\ref{fig:slidingpcpprivgraph} can be used in solving many matrix analysis problems with better additive error.  As a warm up, we consider directional variance queries.  Theorem~\ref{thm:covariance} is true for any $d$-dimensional unit vector $x \in \R^d$. However, in many practical scenarios, it is infeasible to ask all possible questions, but only a bounded number of queries. If we are given an apriori bound $q$ on the number of queries an analyst can make in the entire history, we can apply {\scshape Sliding-Priv-Approx} to get dimension independent bound. 

\begin{theorem}
[Directional variance queries]
\label{thm:covariancelimited}
Given privacy parameters $(\epsilon,\delta)$ and approximation parameter $\eta \in (0,1/2)$, let $A_W$ be the matrix formed by last $W$ updates as defined in equation~(\ref{eq:sliding}). Given a bound $q$, on the number of queries that can be made, there is an efficient $(\epsilon,\delta)$-differentially private algorithm that outputs a matrix $C$ such that for any set of $q$ unit vector queries $x_1, \cdots, x_q \in \R^d$, we have for all $i \in [q]$
\[
 \ip{x_i A_w}{A_w x_i}  - \frac{c \log q \log d}{\epsilon} \leq \ip{x_i}{C x_i} \leq \frac{1}{(1-\eta)} \ip{x_i A_w}{A_w x_i} + \frac{c \log q \log d}{\epsilon}.
\]
\end{theorem}
\begin{proof}
We use $r = O(\log q)$ in Algorithm~\ref{fig:slidingpcpprivgraph} to get a matrix $\widetilde A$. Then we output $C = \widetilde A^\top \widetilde A$. Corollary~\ref{cor:approximation} then gives us the result. 
\end{proof}

For the rest of this section, we will explore applications of Algorithm~\ref{fig:slidingpcpprivgraph} in principal component analysis in both constrained and unconstrained form. Recall that Theorem~\ref{thm:pca} gives an accuracy bound on PCA that depends linearly on the dimension of the data. However, one would ideally like the dependencies to be sublinear in $d$. In this section, we give a method to achieve this. 

For this, we introduce a new concept which we call {\em private projection preserving summary}. This can be seen as the private analogue of projection cost preserving sketches~\cite{cohen2015dimensionality}. However, it is far from clear if the techniques used in non-private literature extends straightforwardly. We show that Algorithm~\ref{fig:slidingpcpprivgraph} with proper choice of parameters provides us with one such summary. This can be later employed to solve constrained and unconstrained principal component analysis. 

\subsection{Private Projection  Preserving Sketch}
\label{sec:projectionpreservingsummary}
In this section, we introduce a  notion called {\em private projection preserving summary}. This notion would be useful in giving our bounds on the principal component analysis and in giving the first bound on private constrained principal component analysis.
\begin{definition}
[Private Projection  Preserving Summary]
\label{defn:ppcp}
Let $k$ be a desired rank. Given a set of rank-$k$ projection matrices $\Pi$, a matrix $\widetilde A \in \R^{n\times d}$ is called a private projection preserving summary of $A \in \R^{n \times d}$ with error $0\leq \eta < 1$ if it is $(\epsilon,\delta)$ differentially private and for all $P \in \Pi$,
\[
(1-\eta) \norm{A - PA}_F^2 - \tau_1 \leq \norm{\widetilde A - P\widetilde A}_F^2  \leq (1-\eta) \norm{A - PA}_F^2 + \tau_2
\]
for $\tau_1, \tau_2$ that depends on $k,d, \epsilon,\delta$. 
\end{definition}

Private projection preserving summary can be seen as the private analogue of projection-cost preserving sketches introduced by Cohen, Elder, Musco, Musco, and Persu~\cite{cohen2015dimensionality}; however, there are some key differences. First of all, we do not require a constant $c$ that depends only on $A$ and $\widetilde A$. Second of all, there is a difference in quantifier. We do not require the private projection preserving summary to be with respect to all rank-$k$ projection matrices but a predefined set $\Pi$ of projection matrices. This flexibility is important as we can design the projection cost preserving summary that depends on the projection matrices. However, our result is independent of the set of projection matrices. In particular, we show the following:
\begin{lemma}
\label{lem:slidingPCP}
Let $k$ be the desired rank, $\eta$ be the approximation parameter, and $(\epsilon,\delta)$ be the privacy parameter. Let $\Pi$ be the set of all rank-$k$ projection matrices. Then for a given matrix $A_W$ formed by the last $W$ updates as defined in equation~(\ref{eq:sliding}), the output $\widetilde A \leftarrow${\scshape Sliding-Priv-PCP}$\paren{\Omega;\frac{k+\log(1/\beta)}{\eta}; (\epsilon,\delta); W}$ is $(\epsilon,\delta)$-differentially private and satisfies the following for all $P \in \Pi$ with probability at least $1-\beta$,
\[
(1-6\eta) \norm{ \widetilde A (\I_d - P) }_F - c_1 K \leq \norm{A (\I_d - P)}_F \leq (1+ 6\eta) \norm{ \widetilde A (\I_d - P) }_F + c_2 K,
\]
where 
\[
K := \frac{\sqrt{k' d \log(1/\delta)}}{\epsilon}  \qquad \text{for} \qquad k':= k +\log(1/\beta).
\]
\end{lemma}
\begin{proof}
Let 
\[
\widehat A = \begin{pmatrix} \sigma \I_d \\ A \end{pmatrix} 
\qquad \text{where} \qquad 
\sigma = \frac{16\sqrt{r \log(4/\delta)} + \log(4/\delta)}{\epsilon}
\] 
is as defined in Algorithm~\ref{fig:slidingprivpcpupdate}. By subadditivity of Frobenius norm, we have 
\begin{align}
    \norm{A(\I_d - P)}_F - \frac{\sqrt{k' d \log(1/\delta)}}{\epsilon} \leq \norm{\widehat A (\I_d - P)}_F \leq \norm{A(\I_d - P)}_F + \frac{\sqrt{k' d \log(1/\delta)}}{\epsilon}.
    \label{eq:subadditive}
\end{align}

We wish to use Lemma~\ref{lem:sufficientpcp}. For this, we design $E_1, E_2, E_3$ and $E_4$ that satisfy the conditions of Lemma~\ref{lem:sufficientpcp}. Let $\widehat A = \widehat V \widehat S \widehat U^\top$ be a singular value decomposition. Let $\widehat U_k$ be the top $k$-left singular vectors of $\widehat A$ and let $P_1 = {\widehat U}_k {\widehat U}_k^\top$ be a rank-$k$ orthonormal projection matrix. Then we define 
\[
E_1 :=   P_1^\top \widehat A^\top \Phi^\top \Phi \widehat A P_1  -  P_1^\top \widehat A^\top \widehat A P_1.
\]
From the {\em subspace embedding property} of $\Phi$~\cite{clarkson2017low}, we have that 
\[
-\eta_1 \widehat A^\top \widehat A \preceq E_1 \preceq \eta_1 \widehat A^\top \widehat A.
\]

\noindent Let $P_2 := (\I_d - P_1)$ and define
\[
E_2 :=   P_2^\top \widehat A^\top \Phi^\top \Phi \widehat A P_2  -  P_2^\top \widehat A^\top \widehat A P_2.
\]
By construction $E_2$ is symmetric. Moreover, using~\citet[Theorem 21]{kane2014sparser} gives us
\begin{align*}
    \tr{E_2} &= \tr{  P_2^\top \widehat A^\top \Phi^\top \Phi \widehat A P_2 -  P_2^\top \widehat A^\top \widehat A P_2 } = \tr{ P_2^\top \widehat A^\top ( \Phi^\top \Phi - \I_d)  \widehat A P_2} \\
        & \leq \norm{ \widehat A P_2}_F \norm{\Phi^\top \Phi - \I_d}_2 \norm{ \widehat A P_2}_F  = \norm{\widehat A -  \widehat A P_1}_F \norm{\Phi^\top \Phi - \I_d}_2 \norm{\widehat A -  \widehat A P_1}_F \\
        & = \norm{\widehat A -  [\widehat A]_k}_F^2 \norm{\Phi^\top \Phi - \I_d}_2 \leq \frac{\eta}{k} \norm{\widehat A - [\widehat A]_k}_F^2
\end{align*}
because $P_2$ is the orthogonal projection to the top-$k$ singular space of $\widehat A$. Using the inequality relationship between trace norm and Frobenius norm, it follows that
\[
\sum_{i=1}^k |\lambda_i(E_2)| \leq \sqrt{k} \norm{E_2}_F.
\]

\noindent Now we define the matrices $E_3$ and $E_4$ using the cross terms as below.
\begin{align*}
 E_3 :=  P_1^\top \widehat A^\top \Phi^\top \Phi  \widehat A P_2 -  P_1^\top \widehat A^\top  \widehat A P_2  
 \qquad \text{and} \qquad
 E_4 :=  P_2^\top \widehat A^\top \Phi^\top \Phi  \widehat A P_1 -  P_2^\top \widehat A^\top \widehat A P_1.
\end{align*}

\noindent We show the desired bound for the case of $E_3$. The case for $E_4$ follows similarly. First of all, by definition of $P_1$ and the fact that $P_1$ and $P_2$ are orthogonal to each other, we simply have 
\begin{align}
    E_3 &=  P_1^\top \widehat A^\top \Phi^\top \Phi  \widehat A P_2 -  P_1^\top \widehat A^\top  \widehat A P_2  =  P_1^\top \widehat A^\top \Phi^\top \Phi  \widehat A P_2
\end{align}

Using~\citet[Theorem 21]{kane2014sparser}, if follows from the singular value decomposition of $\widehat A$ and $\norm{\widehat U_k}_F = \sqrt{k}$, that
\begin{align}
\begin{split}
    \tr{E_3^\top (\widehat A^\top \widehat A)^\dagger E_3} 
    &= \norm{ \widehat U_k \Phi^\top  \Phi \widehat A (\I - P_1)  }_F^2 \\
    &\leq \eta^2 \norm{\widehat A - [\widehat A]_k}_F^2
\end{split}    
\end{align}
as required. Using Lemma~\ref{lem:sufficientpcp}, we have
\begin{align*}
(1-5\eta) \norm{ \Phi \widehat A(\I_d - P)}_F \leq \norm{\widehat A (\I_d - P)}_F + \eta \norm{\widehat A - [\widehat A]_k}_F \leq (1+5\eta) \norm{\Phi\widehat A (\I_d - P)}_F.
\end{align*}
Rearranging the term and using Equation~(\ref{eq:subadditive}) completes the proof of Lemma~\ref{lem:slidingPCP}.
\end{proof}

\subsection{Private Principal Component Analysis}
\label{sec:privatePCA}
In Section~\ref{sec:applicationwishart}, we gave an $(\epsilon,\delta)$-differentially private algorithm under sliding window model to compute the principal component of the matrix formed in the sliding window. However, the additive error there depends linearly on the dimension. In this section, we show that we can improve the accuracy guarantee significantly. Let $\Pi$ be the set of all rank-$k$ orthonormal projection matrices. We use the following algorithm for this purpose: 

\begin{algorithm}[htbp]
\caption{{\scshape Sliding-Priv-PCA}$(\Omega, k;(\epsilon,\delta);W)$}
\label{fig:slidingpcaprivate}
\begin{algorithmic}[1]
\Require{A stream, $\Omega$, of row vectors $\set{a_t}$ and the desired rank of matrices, $r$, privacy parameters $(\epsilon,\delta)$, and window size $W$.} 
\Ensure{A real-valued matrix $ {\widetilde A}$ at the end of the stream.}

\State {\bf Compute} \Comment{Algorithm~\ref{fig:slidingpcpprivgraph}}
\[
\widetilde A \leftarrow \text{{\scshape Sliding-Priv-Approx}} \paren{\Omega;\frac{k+\log(1/\beta)}{\eta}; (\epsilon,\delta); W}.
\]

\State {\bf Output} \[
        \widetilde X := \argmin_{P, \mathsf{rank}(P) \leq k} \norm{  \widetilde A (\I_d -P)}_F.
        \]
\end{algorithmic}
\end{algorithm}

As before, we can solve the rank-constrained problem using the result of~\citet{upadhyay2018price} (see Section~\ref{sec:rankconstrained}).

\begin{theorem}
\label{thm:pcaimprovedspace}
Given privacy parameters $(\epsilon,\delta)$ and approximation parameter $\eta \in (0,1/2)$, let $A_W$ be the matrix formed by the last $W$ updates  as defined in equation~(\ref{eq:sliding}) and $\Pi$ be the set of all rank-$k$ orthonormal projection matrices. Then {\scshape Sliding-Priv-PCA}, described in Algorithm~\ref{fig:slidingpcaprivate}, is an efficient $(\epsilon,\delta)$-differentially private algorithm that outputs a rank-$k$ orthonormal projection matrix $\widetilde X \in \R^{d \times d}$ such that with probablity $1-\beta$, 
\[
\norm{A_W - A_W\widetilde X}_F \leq \paren{\frac{1+6\eta}{1-6\eta}} \min_{P \in \Pi} \norm{A_W(\I_d - P)}_F + O\paren{\frac{\log(1/\delta)}{\epsilon} \sqrt{ (d \log d)(k+\log(1/\beta))}}.
\]
\end{theorem}
\begin{proof}
Let $k ' = k + \log(1/\beta)$. Define
\[
\widehat X := \argmin_{X \in \Pi} \norm{A_W(\I_d - P)}_F
\qquad \text{and} \qquad
\widetilde X := \argmin_{X \in \Pi} \norm{\widetilde A(\I_d - P)}_F.
\]

\noindent Using the left hand inequality in Lemma~\ref{lem:slidingPCP}, we have
\begin{align}
    \begin{split}
        \norm{\widetilde A(\I_d - \widetilde X)}_F & \leq \norm{\widetilde A(\I_d - \widehat X) }_F \leq \frac{1}{(1-6\eta)} \paren{ \norm{ A_W (\I_d - \widehat X) }_F + \frac{k'd \log (1/\delta)}{\epsilon} }.
    \end{split}
    \label{eq:AtildeminA} 
\end{align}

\noindent Similarly, using the right hand inequality in Lemma~\ref{lem:slidingPCP}, we have 
\begin{align}
    \begin{split}
        \norm{\widetilde A(\I_d - \widetilde X)}_F & \geq  \frac{1}{(1+6\eta)} \paren{ \norm{ A_W (\I_d - \widetilde X) }_F - \frac{k'd \log (1/\delta)}{\epsilon} }.
    \end{split}
    \label{eq:AtildeA_W} 
\end{align}

\noindent Combining Equations~(\ref{eq:AtildeminA}) and~(\ref{eq:AtildeA_W}), we have the result.
\end{proof}


\subsection{Private Constrained Principal Component Analysis}
\label{sec:privateconstrainedPCA}
In this section, we prove our result on private constrained principal component analysis. The traditional notion of principal component analysis minimizes over all possible set of rank-$k$ projection matrices. However, recently researchers in machine learning has found applications where the rank-$k$ projection matrices also satisfy other structural properties, like sparsity and non-negativity~\cite{asteris2014nonnegative, d2005direct, papailiopoulos2013sparse, yuan2013truncated, zass2007nonnegative}. There are also non-private algorithms for such problems as well. On the other hand, there is no prior work in privacy preserving literature for structural principal component analysis. 

\begin{algorithm}[htbp]
\caption{{\scshape Sliding-Priv-Constrained-PCA}$(\Omega, k;(\epsilon,\delta);W; \Pi)$}
\label{fig:slidingconstrainedpcapriv}
\begin{algorithmic}[1]
\Require{A stream, $\Omega$, of row vectors $\set{a_t}$ and the desired rank of matrices, $r$, privacy parameters $(\epsilon,\delta)$, a set of rank-$k$ projection matrices, $\Pi$, and window size $W$.} 
\Ensure{A real-valued matrix $ {\widetilde A}$ at the end of the stream.}

\State {\bf Compute} \Comment{Algorithm~\ref{fig:slidingpcpprivgraph}}
\[
\widetilde A \leftarrow \text{{\scshape Sliding-Priv-Approx}} \paren{\Omega;\frac{k+\log(1/\beta)}{\eta}; (\epsilon,\delta); W}.
\]

\State {\bf Output} \[
        \widetilde X := \argmin_{P \in \Pi} \norm{  \widetilde A (\I_d -P)}_F.
        \]
        using any known non-private algorithm. 
\end{algorithmic}
\end{algorithm}

We show the following for constrained principal component analysis: 

\begin{theorem}
\label{thm:constrainedLRA}
Given privacy parameters $(\epsilon,\delta)$ and approximation parameter $\eta \in (0,1/2)$, let $A_W$ be the matrix formed by the last $W$ updates  as defined in equation~(\ref{eq:sliding}) and $\Pi$ be a given set of rank-$k$ orthonormal projection matrices. Algorithm {\scshape Sliding-Priv-Constrained-PCP}, described in Algorithm~\ref{fig:slidingconstrainedpcapriv}, outputs $\widetilde A$ satisfying  $(\epsilon,\delta)$-differential privacy. Let $P \in \Pi$ be a projection matrix satisfying
\begin{align}
    \norm{\widetilde A (\I_d -  P )}_F \leq \gamma \cdot \min_{X \in \Pi} \norm{\widetilde A (\I_d -  X)}_F.
    \label{eq:gamma_approximation}    
\end{align}
Then with probablity $1-\beta$, 
\begin{align*}
\norm{A_W (\I_d - P)} \leq \paren{\frac{1+\eta}{1-\eta}} \gamma \cdot \min_{X \in \Pi} \norm{A_W (\I_d - X)} + O\paren{\frac{\sqrt{k'd }\log(1/\delta)}{\epsilon}  }.
\end{align*}
\end{theorem}

\begin{proof}
Define
\[
\widehat P = \argmin_{P \in \Pi} \norm{A_W (\I_d - P)}_F, \quad 
\widetilde P = \argmin_{P \in \Pi} \norm{\widetilde A (\I_d - P)}_F.
\]
Since $\widetilde P$ is a minimizer,  it implies the following 
\begin{align}
    \norm{\widetilde A (\I_d - \widetilde P) }_F \leq \norm{\widetilde A (\I_d - \widehat P)}_F
\end{align}
as $\widehat P \in \Pi$. Combining with the equation~(\ref{eq:gamma_approximation}), we have 
\begin{align}
    \norm{\widetilde A (\I_d - P)}_F \leq \gamma  \norm{\widetilde A (\I_d - \widetilde P )}_F.
    \label{eq:restrainedtildeAupper}
\end{align}
for $P \in \Pi$ as defined in equation~(\ref{eq:gamma_approximation}). 

\noindent Using the right hand side inequality of Lemma~\ref{lem:slidingPCP}, we have
\begin{align}
     (1+\eta)\norm{\widetilde A (\I_d - P)}_F + \frac{\sqrt{k'd \log(1/\delta)}}{\epsilon} \geq  \norm{ A_W(\I_d - P)}_F 
     \label{eq:restrainedtildeAlower}
\end{align}
Combining Equation~(\ref{eq:restrainedtildeAupper}) and~(\ref{eq:restrainedtildeAlower}), we have 
\begin{align}
    \norm{A_W(\I_d - P)}_F \leq (1+\eta)\gamma \norm{\widetilde A (\I_d - \widetilde P )}_F + \frac{\sqrt{k'd \log(1/\delta)}}{\epsilon}.
\end{align}

Using the left hand side inequality of Lemma~\ref{lem:slidingPCP}, we have 
\begin{align}
  \norm{\widetilde A (\I_d - \widehat P )}_F \leq \paren{\frac{1}{1-\eta}} \norm{ A_W (\I_d - \widehat P )}_F +  \frac{\sqrt{k'd \log(1/\delta)}}{\epsilon(1-\eta)}.
  \label{eq:restrainedAlower}
\end{align}
Combining Equations~(\ref{eq:restrainedAlower}) and~(\ref{eq:restrainedtildeAlower}), we have the result.
\end{proof}

\subsection{Generalized Overconstrained Linear Regression}
The question of generalized linear regression problem is as follows: given two matrices $A\in \R^{n \times d}$ and $B \in \R^{n\times p}$ as input, generalized linear regression is defined as the following minimization problem:
\[
 \argmin_{X \in \R^{d \times p}} \norm{AX - B}_F^2. 
\]
This problem is also known as {\em multiple-response regression}, and is used mostly in the analysis of low-rank approximation (see~\cite{clarkson2017low} and references therein). We show the following: 

\begin{theorem}
[Linear regression]
\label{thm:regression}
Given privacy parameters $(\epsilon,\delta)$ and approximation parameter $\eta \in (0,1/2)$, let $A_W \in \R^{W \times d}$ and $B \in \R^{W \times p}$ be the matrix streamed during the window of size $W$ formed as defined in equation~(\ref{eq:sliding}). Let 
\[
 \widehat X := \argmin_{X \in \R^{d \times p}} \norm{A_WX - B_W}_F^2. 
\]
Then we can efficiently output a matrix $\widetilde X \in \R^{d \times p}$ while preserving $(\epsilon,\delta)$-differential privacy such that 
\[
\norm{A_W \widetilde X - B_W}_F^2 \leq \paren{\frac{1}{1-\eta}} \min_{X \in \R^{d \times p}} \norm{A_WX - B}^2_F + \frac{c(\tau+p)^2 \log (\tau + p)}{\epsilon}
\]
for some large constant $c>0$ and $\tau = d + \frac{14}{\epsilon^2} \log (4/\delta)$. 
\end{theorem}
\begin{proof}
To reduce the notation overhead, we use $A = A_W$ and $B=B_W$ to denote the matrices streamed in the last $W$ updates. We use $\begin{pmatrix} A & B \end{pmatrix} \in \R^{W \times (d+p)}$ as the matrix that is being streamed and assume $C^\top C \in \R^{(d+p) \times (d+p)}$ as an output. We solve the following unconstrained minimization problem:
\begin{align*}
    \text{minimize:} \quad  \tr{\begin{pmatrix} X^\top & -\I_p^\top  \end{pmatrix} C^\top C \begin{pmatrix} X \\ -\I_p \end{pmatrix}} 
\end{align*}

\noindent where $X \in \mathbb{R}^{d \times p}$. Let $\widetilde X$ be an optimal solution to the above minimization problem. In particular, this implies that for all $X \in \R^{d \times p}$,
\begin{align*}
    \tr{ \begin{pmatrix} \widetilde X^\top & -\I_p  \end{pmatrix} C^\top C \begin{pmatrix} \widetilde X \\ -\I_p \end{pmatrix}} \leq \tr{ \begin{pmatrix} X^\top & -\I_p  \end{pmatrix} C^\top C \begin{pmatrix} X \\ -\I_p \end{pmatrix}}. 
\end{align*}

\noindent Let 
\[
\widehat X := \argmin_{X \in \R^{d \times p}} \norm{A_W X - B_W}_F^2
\]
be an optimal solution for the original regression problem. If follows that for all $X \in \R^{d \times p}$,
\begin{align}
    \norm{A_W \widehat X - B_W}_F^2 \leq \norm{A_W X - B_W}_F^2. 
    \label{eq:regressionminimizer}
\end{align}

\noindent The right side of Equation~(\ref{eq:approxfinal}) gives us 
\begin{align}
    \norm{C \begin{pmatrix} \widetilde X \\ -\I_p \end{pmatrix} }_F^2 
    &\leq \norm{C \begin{pmatrix} \widehat X \\ -\I_p \end{pmatrix} }_F^2 \nonumber \\
    &\leq \frac{1}{(1-\eta)} \norm{ \begin{pmatrix} A_W & B_W \end{pmatrix} \begin{pmatrix} \widehat X \\ -\I_p \end{pmatrix} }_F^2  + \frac{c(\tau + p)^2 \log (\tau + p)}{\epsilon} \nonumber \\
    &= \frac{1}{1-\eta} \norm{A_W \widehat X - B_W}_F^2 + \frac{c(\tau + p)^2 \log (\tau + p)}{\epsilon} \nonumber \\
    &= \frac{1}{1-\eta} \min_{X \in \R^{d \times p}} \norm{A_W X - B_W}_F^2 + \frac{c(\tau + p)^2 \log (\tau + p)}{\epsilon}, \label{eq:regressionupper}
\end{align}
where the second equality is by the definition of $\widehat X$. Similarly, the left hand side inequality of Equation~(\ref{eq:approxfinal}) gives us a lower bound as follows:
\begin{align}
    \norm{C \begin{pmatrix} \widetilde X \\ -\I_p \end{pmatrix} }_F^2 
    &\geq \norm{ \begin{pmatrix} A_W & B_W \end{pmatrix} \begin{pmatrix} \widetilde X \\ -\I_p \end{pmatrix} }_F^2  - \frac{c(\tau + p)^2 \log (\tau + p)}{\epsilon} \nonumber \\
    & = \norm{  A_W \widetilde X  - B_W  }_F^2  - \frac{c(\tau + p)^2 \log (\tau + p)}{\epsilon}.     \label{eq:regressionlower}
\end{align}

\noindent Combining Equation~(\ref{eq:regressionupper}) and~(\ref{eq:regressionlower}) gives us the claimed result.
\end{proof}


\section{The big picture and open questions}
\label{sec:future}
We believe that our approach will find applications beyond what is covered in this paper. It also paves way for more research in the intersection of differential privacy and sliding window model. Motivated by heuristics used in recent deployments, the focus of this paper is on the model where every data in the current window is considered equally useful. However, one can alternatively consider other variants of the sliding window model as far as privacy is concerned. For example, one can consider a model where the privacy of a data decays as a monotonic function as time lapse. More so, there are more concrete questions to be asked and answered even in the model studied in this paper. We mention few of them next.  

As we mentioned earlier, one can see $\eta$-approximate spectral histogram property as a generalization of subspace embedding property. We also designed a data structure that allows us to maintain a small set of positive semidefinite matrices that satisfy $\eta$-approximate spectral histogram property on every new update. We believe any improvement in designing a more efficient data structure for maintaining a set of matrices satisfying $\eta$-approximate spectral histogram property would have a profound impact on large-scale deployment of privacy-preserving algorithms in the sliding window model. For example, we believe we can reduce space requirements using randomization. This randomization can be either oblivious or may depend on the current set of positive semidefinite matrices. Since our set of positive semidefinite matrices are generated using a privacy mechanism, any such sampling can be seen as post-processing.  Our main conjectures thus are concerning the space required by any privacy-preserving algorithm. We elaborate them next.

The lower bound of $\Omega(d^2)$ space for spectral approximation is required even in the static setting. We conjecture that there should be $\frac{1}{\eta} \log W$ factor due to the sliding window requirement. This is because, if the spectrum of a matrix is polynomially bounded, then one can construct a sequence of update that requires at least $\frac{1}{\eta}\log W$ matrices such that successive matrices are $(1-\eta)$ apart in terms of their spectrum. For an upper bound, we believe randomization can help reduce a factor of $d$. This is achieved in the non-private setting using online row sampling. It was shown by \citet{upadhyay2018price} that one can design private algorithms with space-bound comparable to a non-private algorithm in the streaming model of computation.  The situation in the sliding window model is more complicated, but we believe more sophisticated techniques can be used to achieve the matching upper bound. That is, we conjecture the following:
\begin{conjecture}
The space required for differentially private spectral approximation is $\Theta\paren{\frac{d^2}{\eta}\log W}$. 
\end{conjecture}

We believe that the bound on the additive error is optimal. So, a positive resolution to this conjecture would imply that the price of privacy is only in terms of the additive error. 

Our second conjecture is for principal component analysis. We believe that our space-bound for principal component analysis is tight up to a factor of $\frac{k}{\eta}$. A lower bound of $\Omega(dk)$ is trivial as one requires $O(dk)$ space just to store the orthonormal matrix corresponding to the rank-$k$ projection matrix. As before, a factor of  $\frac{1}{\eta}\log W$ would be incurred due to the sliding window model. The factor of $\frac{1}{\eta}$ comes from the fact that to extract the top-$k$ subspace, we need $\frac{k}{\eta}$ dimensional subspace. We summarize it in the following conjecture: 
\begin{conjecture}
The space required for differentially private principal component analysis is $\Omega \paren{\frac{dk}{\eta^2}\log W}$.  
\end{conjecture}
We believe that proving such a lower bound would require new techniques. This is because, in principal component analysis, we only have access to an orthonormal projection matrix, while in the case of low-rank approximation, we have far more information to solve the underlying communication complexity problem. 

Our work identifies another application of the Johnson-Lindenstrauss and Wishart mechanisms. Before our results, it was not even clear whether the JL mechanism can be used to compute PCA (see Section V in~\citet{blocki2012johnson})! They consider their output matrix $\widetilde C$ as a ``test" matrix to test if the input matrix has high directional variance along some direction $x \in \R^d$. However, they do not give any guarantee as to how the spectrum of $C$ relates to that of the input covariance matrix. 

\paragraph{Acknowledgements.} The authors would like to acknowledge an anonymous reviewer for pointing out that Cohen et al.~\cite{cohen2015dimensionality} also study constrained rank approximation. JU would like to thank Petros Drineas for suggesting the question of private matrix analysis during the course of the project in the non-private setting~\cite{braverman2018numerical}.

\bibliography{privacy}
\bibliographystyle{alpha}

\begin{appendix}
\section{Related Work}
\label{sec:related}
To the best of our knowledge, differential privacy in the sliding window setting has not been studied for matrix-valued functions. However, the matrix problems considered in this paper have been studied from a privacy perspective in static setting. We first review private matrix analysis in static setting followed by matrix analysis in sliding window model.

\paragraph{Spectral approximation.} Blocki et al.~\citet{blocki2012johnson} initiated the study of private approximate spectral analysis of matrices. For a matrix $A$, they first compute the singular value decomposition (SVD) $U S V^\top$ and then perturb the singular values to compute $\widehat A:= U \sqrt{S^2 + \sigma^2 \I_d} V^\top$. Here $\sigma$ is the perturbation parameter chosen appropriately. Their final output is $C_{\mathsf{BBDS}} = \widehat A^\top \Phi^\top \Phi \widehat A$, where $\Phi$ is a random Gaussian matrix. This is in contrast with~\citet{dwork2014analyze} who compute $C_{\mathsf{DTTZ}} = A^\top A + N$, where $N$ is a symmetric Gaussian matrix with variance required to preserve differential privacy. It is easy to see that both $C_{\mathsf{BBDS}}$ is positive semidefinite and $C_{\mathsf{DTTZ}}$ is a symmetric matrix. Moreover, for $\Phi$ with appropriate dimension, $C_{\mathsf{BBDS}}$ is a spectral approximation of $A$ up to a small distortion in spectrum~\cite{sarlos2006improved}. However, it is not clear if they can be extended to the streaming model or more restricted sliding window model because they first compute and perturb the SVD of $A$\footnote{Dwork et al.~\citet{dwork2014analyze} gave an online algorithm to private singular value computation using regularized follow-the-leader framework; however, online setting is very different from the setting of sliding window model.}. 


\paragraph{Principal component analysis and linear regression.} There has been an extensive body of work on private principal component analysis both Frobenius norm~\cite{blum2005practical, dwork2014analyze, hardt2012beating, upadhyay2018price} and spectral norm~\cite{dwork2014analyze, hardt2014noisy, kapralov2013differentially}. Moreover, in static setting, matching lower and upper bounds are known on achievable accuracy. Except for~\citet{upadhyay2018price}, these results do not extend to dynamic setting and none of them extend to the sliding window model. Finally, private linear regression has also been studied extensively in static setting, and matching upper and lower bounds are known  ~\cite{dwork2010differential, kifer2012private, chaudhuri2011differentially}. 

\paragraph{Non-privare dynamic setting.} In the non-private setting,  Braverman, Drineas, Upadhyay, Woodruff, and Zhou~\cite{braverman2018numerical} noted that many matrix analysis problems do not comply with the existing framework of the sliding window model. They gave an algorithm for spectral sparsification using {\em online row sampling} algorithm~\cite{CohenMP16} and for computing low-rank approximation using {\em projection cost preserving sketches}~\cite{cohen2015dimensionality}. They employed {\em importance sampling} as a subroutine, where each row is sampled with probability defined by the current matrix and the new row. Since the sampling probability depends on the matrix itself, it is not immediately clear how to perform such sampling privately -- sampling probability can itself leak privacy. Upadhyay~\cite{upadhyay2019sublinear} recently showed that, in general, non-private algorithms do not directly extend to the private setting in the sliding window model for even simple tasks such as estimating {\em frequency moments}.

Conceptually  $\eta$-approximate spectral histogram property is a generalization of {\em subspace embedding property} to the sliding window model when the subspace is the intrinsic dimension of the matrix. While subspace embedding is useful when entire historical data is used, it is not clear how to use it for differential privacy in the sliding window model.  The underlying reasons are as follows. Subspace embedding has been used extensively in performing randomized numerical linear algebra~\cite{DrineasM17, Woodruff14} when the entire historical data is considered important. 
There are two broad techniques for subspace embedding: {\em row-sampling} and {\em random sketching}. While row-sampling can be extended to the sliding window model~\citet{braverman2018numerical}, it is not immediate if we can perform such sampling privately because sampling probability can itself lead to a loss in privacy. On the other hand, while we can generate private sketches using random projection~\citet{blocki2012johnson, mir2011pan, upadhyay2018price}, it is unclear how to remove from the current sketch the contribution of rows before the last $W$ updates.

\section{Auxiliary lemma}
\label{sec:rankconstrained}
The following is a well known result in randomized numerical linear algebra, first shown by Sarlos~\citet{sarlos2006improved} and later improved by a series of works. 
\begin{theorem}
[Clarkson-Woofruff~\citet{clarkson2017low}]
\label{thm:sketch}
Let $A \in \R^{n \times d}$ be a rank $r$ matrix and $\Phi \in \R^{n \times \frac{4r}{\eta}}$ be random matrix with i.i.d. copies of $\mathcal N(0,\frac{\eta}{4r})$. Then we have 
\[
\pr \sparen{ \left(1 -\frac{\eta}{4} \right) A^\top A \preceq A^\top \Phi^\top \Phi A \preceq \left(1 + \frac{\eta}{4} \right) A^\top A} \geq 1 - \frac{1}{\poly(d)}.
\]
\end{theorem}

\begin{definition}
[Subspace embedding]
An {\em embedding} for a set of points $P \subseteq \R^n$
 with distortion $\alpha$ is an $r \times n$ matrix $E$ such that
\[
\forall x \in P, (1-\eta) \norm{x}_2^2 \leq \norm{Ex}_2^2 \leq (1+\eta) \norm{x}_2^2.
\]
A {\em subspace embedding} is an embedding for a set $K$, where $K$ is a $k$-dimensional linear subspace.

\end{definition}

In this paper, we use random projections. 
For more details, see Clarkson and Woodruff~\cite{clarkson2017low}.
\begin{lemma}
[Sarlos~\citet{sarlos2006improved} and Clarkson-Woodruff~\citet{clarkson2017low}]
\label{lem:phi}
Let $\mathcal R \in \R^{t \times n}$ be a distribution of random Gaussian matrix, i.e., for $R \sim \mathcal R$, such that $R[i,j] \sim \cN(0,1/t)$. Then $R$ satisfies subspace embedding for some $t = O(\eta^{-2} \log (1/\beta))$ and affine embedding for some $t = O(r\eta^{-2} \log (1/\beta))$.
\end{lemma}

A key concept in randomized numerical linear algebra and low rank approximation is that of rank-$k$ projection cost preserving sketch. It was defined by \citet{cohen2015dimensionality}.  
\begin{definition}
[Rank-$k$ projection-cost preserving sketch~\cite{cohen2015dimensionality}] A matrix $\widetilde A \in \R^{n\times d}$ is a rank $k$ projection-cost preserving sketch of $A \in \R^{n \times d}$ with error $0\leq \eta < 1$ if, for all rank $k$
orthogonal projection matrices $P \in \R^{n\times n}$,
\[
(1-\eta) \norm{A - PA}_F \leq \norm{\widetilde A - P\widetilde A}_F  + c \leq (1+\eta) \norm{A - PA}_F
\]
for some fixed non-negative constant $c$ that may depend on $A$ and $\widetilde A$ but is independent of $P$.
\end{definition}

Cohen et al.~\citet{cohen2015dimensionality} showed many characterizations of rank-$k$ projection cost preserving sketch. One of their characterizations is relevant to our work which we present in the lemma below.
\begin{lemma}
[Sufficient condition for rank-$k$ PCP~\cite{cohen2015dimensionality}]
\label{lem:sufficientpcp}
A matrix $\widetilde A$ is a rank-$k$ projection-cost preserving sketch with two sided error $\eta$ and $c = \min\set{0, \eta_2' \norm{A - [A]_k}_F}$ as long as we can write 
$$\widetilde A^\top \widetilde A - A^\top A = E_1 + E_2 + E_3 + E_4,$$ where 
\begin{enumerate}
    \item $E_1$ is symmetric and $-\eta_1 A^\top A \preceq E_1 \preceq \eta_1 A^\top A$.
    \item $E_2$ is symmetric, $\tr{E_2} \leq \eta_2'\norm{A - [A]_k}_F^2$, and 
    $$\sum_{i=1}^k |\lambda_i(E_2)| \leq \eta_2 \norm{A - [A]_k}_F^2.$$
    \item The span of columns of $E_3$ is a subspace of span of columns of $A^\top A$ and 
    $$\tr{E_3^\top (A^\top A)^\dagger E_3} \leq \eta^2_3 \norm{A - [A]_k}_F^2.$$
    \item The span of rows of $E_4$ is a subspace of span of rows of $A^\top A$ and $$\tr{E_4^\top (A^\top A)^\dagger E_4} \leq \eta^2_3 \norm{A - [A]_k}_F^2.$$
    \item $\eta = \eta_1 + \eta_2 + \eta_2' + \eta_3 + \eta_4$. 
\end{enumerate}
\end{lemma}

\begin{proposition}
\label{prop:inequalityeta}
For $\eta \in (0,1)$, we have  $(1-\eta) <  \paren{1-\frac{\eta}{2}} \paren{\frac{1-\frac{\eta}{4}}{1+\frac{\eta}{4}}}$
\end{proposition}
\begin{proof}
For $\eta \in (0,1)$, we have the following:
\[
(1-\eta)\paren{1 + \frac{\eta}{4}} = 1 - \frac{3\eta}{4} - \frac{\eta^2}{4} \le 1 - \frac{3\eta}{4} + \frac{\eta^2}{8} = \paren{1 - \frac{\eta}{2}} \paren{1 - \frac{\eta}{4}}.
\]
Since ${\eta^2} > 0,$ the result follows.
\end{proof}

The following lemma was first proven in \citet{upadhyay2018price}. We give a proof for the sake of completion.
\begin{lemma} \label{lem:orthonormal}
Let $ R $ be a matrix with orthonormal rows and $ C $ have orthonormal columns. Then  for a given matrix $F$ of conforming dimensions, we have 
\[ \min_{X: \mathsf{rank}(X)=k} \| CXR - F \|_F = \| C [C^\top FR^\top ]_k R - F \|_F, \]
where $\mathsf{rank}(X)$ denotes the rank of matrix $X$.
 \end{lemma}
 \begin{proof}
For any matrix $Y $ of appropriate dimension, we have 
\[\brak{F - CC ^\top F  ,  C   C ^\top  F   -  C  Y   R }=0.
\]
This is because $F  - C   C ^\top  F   = ( \I - C   C ^\top ) F  $ lies in space orthogonal to $ C  ( C ^\top  F   -  Y   R )$. By Pythagorean theorem, 
 \begin{align}
 \| F - C  Y   R  \|_F^2 
 	&= \| F   -  C   C ^\top  F   \|_F^2 + \|  C   C ^\top  F   -  C  Y   R  \|_F^2   = \| F   -  C   C ^\top  F   \|_F^2 + \|   C ^\top  F   -  Y   R  \|_F^2, \label{eq:orthonormal1}
\end{align}	 
where the second equality follows from the properties of unitary matrices. 
{Again, for any matrix $Y $ of appropriate dimensions, we have 
\[
\brak{C ^\top FR ^\top R  - YR ,  C^\top F -C ^\top  FR ^\top R} =0.
\]
This follows because $C ^\top FR ^\top R- YR = (C^\top  FR ^\top - Y)R$ lies in the space spanned by $ R $, and $C^\top F- C^\top FR^\top R =C ^\top  F(\I-R^\top R)$ lies in the orthogonal space. Applying Pythagorean theorem again, we have that}
\begin{align}
 \|C^\top F - YR \|_F^2
    	&=  \| C^\top  F - C^\top F R ^\top   R  \|_F^2 
	+ \|  C ^\top  F    R ^\top   R  -Y   R  \|_F^2 \label{eq:orthonormal2}
\end{align}	 

{Since $\|  C ^\top  F   -  C ^\top  F    R ^\top   R  \|_F^2 $ is independent of $Y $, we just bound the term $\|  C ^\top  F    R ^\top   R  -Y   R  \|_F^2$. Substituting $Y  = [  C  F    R ]_k$ and using the fact that multiplying $ R $ from the right does not change the Frobenius norm and $[ C ^\top  F    R ^\top ]_k $ is the best $k$-rank approximation to the matrix $ C ^\top  F    R ^\top $, for all rank-k matrices $Z $, we have } 
 \begin{align}
\|  C ^\top  F    R ^\top   R  -[ C ^\top  F    R ^\top ]_k  R  \|_F^2 &\leq \|  C ^\top  F    R ^\top   R  -Z   R  \|_F^2. \label{eq:orthonormal3}
\end{align}
{Combining~\eqnref{orthonormal3} with~\eqnref{orthonormal2} and Pythagorean theorem, we have}
 \begin{align}
\|   C ^\top  F   -   [  C  F    R ]_k  R  \|_F^2     & \leq \|  C ^\top  F   -  C ^\top  F    R ^\top   R  \|_F^2  + \|  C ^\top  F    R ^\top   R  -   Z   R  \|_F^2= \|  C ^\top  F   -  Z   R  \|_F^2. \label{eq:orthonormal4}
\end{align}
{Combining~\eqnref{orthonormal4} with~\eqnref{orthonormal1}, the fact that $ C $ has orthonormal columns, and Pythagorean theorem, we have}
\begin{align} \| F   -  C  [  C  F    R ]_k  R  \|_F^2 
 	&\leq \| F   -  C   C ^\top  F   \|_F^2 + \|  C ^\top  F   -  Z   R  \|_F^2 \nonumber \\
 	&= \| F   -  C   C ^\top  F   \|_F^2 + \|  C   C ^\top  F   -  C  Z   R  \|_F^2 = \| F    -  C  Z   R  \|_F^2. \nonumber 
 \end{align}
 This completes the proof of~\lemref{orthonormal}.
 \end{proof}

\paragraph{Differential privacy.}
An often easy to handle way to define $(\epsilon,\delta)$-differential privacy is in the terms of privacy loss function. Let $S$ be the support of the output of an algorithm $M$. Let $\mathcal{P}$ be the output distribution of $M$ when its input is $A_W(t)$ and $\mathcal{Q}$ be the output distribution of $M$ when its input is $A_W'(t)$. Then for $v \in S$, we define the privacy loss function as follows:
\[
L(v) := \log \paren{\frac{\mathcal P(v)}{\mathcal{Q}(v)}}.
\]
An algorithm $M$ is $(\epsilon,\delta)$ differentially private if 
\[
\pr_{v \sim \mathcal P} \sparen{L(v;M) \leq \epsilon} \geq 1 - \delta.
\]

    


One of the key features of differential privacy is that it is preserved under arbitrary post-processing, i.e., an analyst, without additional information about the private database, cannot compute a function that makes an output less differentially private. In other words,
\begin{lemma}   
[Dwork et al.~\citet{dwork2006calibrating}]
\label{lem:post}
Let ${M}({D})$ be an $(\epsilon, \delta)$-differential private mechanism for a database ${D}$ , and let $h$ be any function, then any mechanism ${M}':=h({M}({D}))$ is also $(\epsilon,\delta)$-differentially private for the same set of tasks.
\end{lemma}

 
\begin{theorem}
[Wishart mechanism~\cite{sheffet2015private}]
\label{thm:wishartprivacy}
Draw a sample $R \sim \mathsf{Wis}_d(\tau, \I_d)$, where  $\tau \geq d + \frac{28 \ln(4/\delta)}{\epsilon^2}$. Then for a matrix $X \in \R^{n \times d}$, $X^\top X + R$ is $(\epsilon,\delta)$-differentially private. 
\end{theorem}

We also make use of the following result regarding  Johnson-Lindenstrauss mechanism first introduced by Blocki et al.~\citet{blocki2012johnson} and later improved by Sheffet~\citet{sheffet2015private}.
\begin{theorem}
[Johnson-Lindenstrauss mechanism~\cite{sheffet2015private}]
\label{thm:JLmechanism}
Fix a positive integer $r$ and let $w$ be such that
$w = \frac{4 \sqrt{r \log(4/\delta)} + \log(4/\delta)}{\epsilon}.$
Let $A \in \R^{n \times d}$ such that $d < r$ and where the Euclidean norm of each row of $A$ is upper bounded by $1$. Given that $s_d(A) \geq w$, the algorithm that picks an $(r \times n)$-matrix $R$ whose entries are i.i.d samples from  $\mathcal N (0, 1)$ and outputs $RA$ is $(\epsilon,\delta)$-differentially private.
\end{theorem}

Sheffet~\citet{sheffet2015private} showed that adding noise matrix according to Wishart distribution preserves $(\epsilon,\delta)$-differential privacy. Sheffet~\cite{sheffet2015private} considers the setting of Blocki et al.~\cite{blocki2012johnson} and improve it. In Blocki et al.~\cite{blocki2012johnson}, matrices $A$ and $A'$ are neighboring if $A-A'$ is a rank-$1$ matrix with bounded norm (see their abstract). This is our setting, too. Sheffet~\cite{sheffet2015private} considers the presence or absence of an entire row as neighboring relation (in the same manner as Dwork et al.~\cite{dwork2014analyze}). Hence, they have bounded row assumption in their theorem. Their results also hold if we remove bounded row norm assumption and consider $A - A'$ to be rank-$1$ with singular value $1$, i.e., our setting.

To prove our lower bound, we give a reduction to the augmented indexing problem, $\mathsf{AIND}$: 

\begin{definition} ($\mathsf{AIND}$ problem). Alice is given an $N$-bit string $ x$ and Bob is given an index ${\mathsf{ind}} \in [N]$ together with $ x_{{\mathsf{ind}}+1}, \cdots,  x_N$. The goal of Bob is to output $ x_{{\mathsf{ind}}}$. 
\end{definition}

The communication complexity for solving $\mathsf{AIND}$ is well known:  
\begin{theorem}
[Miltersen et al.~\citet{MNSW98}]\label{thm:aind}
The minimum number of bits of communication required to solve $\mathsf{AIND}$ with probability $2/3$ in one way communication model (the messages are sent either from Alice to Bob or from Bob to Alice), is  $\Omega(N)$.  This lower bound holds even if the index ${\mathsf{ind}}$ and the string $ x$ is chosen uniformly at random.
\end{theorem}

\section{Spectral histogram Property and Private Spectral Approximation}
\label{sec:slidingwishart}
The goal of this section is to show that the sufficient condition used in~\citet{braverman2018numerical} only provides sub-optimal accuracy. We also show a simplification of their analysis. For the sake of completion, we state their definition. 

\begin{definition}
[Spectral histogram property~\cite{braverman2018numerical}] 
\label{defn:smoothlaplacian}
A data structure $\mathfrak{D}$ satisfy the {\em spectral histogram} property if there exists an $\ell =\poly(n,\log W)$ such that $\mathfrak{D}$ satisfy the following conditions:
\begin{enumerate}
    
        

    \item $\mathfrak{D}$ consists of $\ell $ timestamps  $\mathsf I:=\{ t_{1} , \cdots, t_\ell \} $ and the corresponding PSD matrices $\mathsf{S}:=\{  K_1,  \cdots,  K_\ell \}.$ 
    
    \item For $1\leq i\leq \ell -1$, at least one of the following holds:
    \begin{enumerate}
        \item If $t_{i+1} = t_{i}+1$, then $\paren{1-\eta}  K_i \not\preceq  K_{i+1}.$ \label{item:technical}
        
        \item \label{item:combined}
        For all $1\leq i \leq \ell -2$:
        \begin{enumerate}
        \item
        $(1-\eta) K_i \preceq  K_{i+1}$.
        \item
        $\paren{1-\eta}  K_{i} \not\preceq  K_{i+2}$.
        \end{enumerate}
        \end{enumerate}
    \item Let $A_W$ be the matrix formed by the window $W$, then 
    $K_2 \preceq A_W^\top A_W \preceq K_1.$
    \label{item:approx}
\end{enumerate}
\end{definition}

 

\begin{algorithm}[htbp]
\caption{{\scshape Update}$(\dnormal)$}
\label{fig:slidingprivupdate}
\begin{algorithmic}[1]

\Require{A data structure $\dnormal$ a set of positive semidefinite matrices $\set{K(1), \cdots, K(\ell)}$ such that $$K(1) \succeq K(2) \succeq  \ldots \succeq K(\ell)$$ and corresponding  a set of timestamps $t_{1}, \cdots, t_\ell$.} 

\Ensure{Updated set of positive semidefinite matrices $ K(1), \cdots, K(\ell)$ and timestamps $t_{1}, \cdots, t_\ell$.}

\State {\bf for} {$i=1, \cdots \ell-2$}
\label{step:stage2start}
\label{step:smoothupdatingmain}
    \State \hspace{5mm} Find     
    \Comment{\small{Find spectrally close checkpoints.}}
    \begin{align}
    j := \max \set{ p  : (1-\eta)  K(i) \preceq K(p) \wedge (i < p \leq \ell-1) }.
    \label{step:checkmain}
    \end{align}

    \State \hspace{5mm} {\bf Delete} $ K(i+1), \cdots,   K(j-1)$. \label{step:deletemain}
    \Comment{\small{It is important that we delete only up to index $j-1$.}}
    
    \State \hspace{5mm} {\bf Set} k=1
    \State \hspace{5mm} {\bf while} {$i+k \leq \ell$}  \label{step:updatemain}
        \State \hspace{1cm} {\bf Update} the checkpoints as follows: $$K(i+k) = K(j+k-1), t_{i+k} = t_{j+k-1}.$$
    \State \hspace{5mm} {\bf end}
    \State \hspace{5mm}  Update $\ell:=\ell+i-j+1$.
\State {\bf end}
\label{step:stage2end}

\State {\bf Return} $\dsmooth := \set{(K(1), t_1), \ldots, (K(\ell), t_\ell)}$.
\end{algorithmic}
\end{algorithm}

We first show existence of an algorithm ({\scshape Update}) that on takes a set of positive semidefinite matrices not necessarily satisfying spectral histogram property as input and outputs a set of positive semidefinite matrices satisfying spectral histogram property. The algorithm {\scshape Update} performs a sequential check over the current set of positive semidefinite matrices and removes all the matrices that do not satisfy the spectral histogram property. Since we make no assumption on the input except that they satisfy Loewner ordering, it is possible that, in the worst case, all but one matrix can be deleted in step~\ref{step:deletemain}. However, as we will see later, {\scshape Update} will form a subroutine of our differentially private algorithms such that the input to  {\scshape Update} will have a particular form on top of satisfying the Loewner ordering. This will help us utilize {\scshape Update} in a much better way. We being with showing the following for {\scshape Update} algorithm


\begin{lemma}
[Spectral histogram property]
\label{lem:smoothpsd}
Let $\dnormal = \set{(K(1),t_1), \ldots, ( K(\ell),t_\ell)}$ be a set consisting of timestamps and positive semidefinite matrices such that 
$$K(1) \succeq K(2) \succeq  \ldots \succeq K(\ell) \succeq 0.$$ 
Let  $\dsmooth \leftarrow${\scshape Update}$(\dnormal)$ be the output of the algorithm defined in Algorithm~\ref{fig:slidingprivupdate}. Then {\scshape Update($\cdot$)} is an efficient algorithm and $\dsmooth$ satisfy spectral histogram property.
\end{lemma}


\begin{proof}
[Proof Sketch of Lemma~\ref{lem:smoothpsd}]
The proof of the above lemma can be derived from Lemma~\ref{lem:smoothLaplacianproperty} which states a more general case of approximation. We give a short proof sketch below. 

Consider a time epoch $T$ and a succeeding time epoch $T'=T+1$. Let the data structure at time $T$ be  $\dpriv(T)$ and at time $T'$ be $\dpriv(T')$. 
Let $t_i$ be a timestamp in $\dpriv(T)$ where $i <\ell$ . We can have two cases: 
    (i) 
    There is no $1 \leq j \leq s $ such that $t_j'=t_i$ and $t_{j+1}' = t_{i+1}$, and 
    (ii) There is a $1 \leq j \leq s $ such that $t_j'=t_i$ and $t_{j+1}' = t_{i+1}$. 
In both cases, spectral histogram property follows from the update rules. This is because we delete indices up to $j-1$ in Step~\ref{step:deletemain} and the maximality of the index $j$ in  Algorithm~\ref{fig:slidingprivupdate}.  
\end{proof}

We next give an intuition why we need this lemma. Lemma~\ref{lem:smoothpsd} gives the guarantee that, if we are given a set of positive semidefinite matrices in Loewner ordering, then we can efficiently maintain a small set of positive semidefinite matrices that satisfy $\eta$-approximate spectral histogram property. The idea of our algorithm for spectral approximation is to ensure that {\scshape Update} always receives a set of positive semidefinite matrices. This is attained by our algorithm {\scshape Priv-Initialize}, described in Algorithm~\ref{fig:slidingprivinitialize}. {\scshape Priv-Initialize} gets as input a new row and updates all the matrices in the current data structure.

\begin{algorithm}[ht]
\caption{{\scshape Priv-Initialize}$(\dpriv; a_t; t; (\epsilon,\delta); W; r)$}
\label{fig:slidingprivinitialize}
\begin{algorithmic}[1]
\Require{A new row $a_t \in \mathbb{R}^d$, a data structure $\dpriv$ storing a set of timestamps $t_{1}, \cdots, t_\ell$ and set of matrices 
$${\widetilde K(1), \cdots,  \widetilde K(\ell+1)},$$ current time $t$, privacy parameters $(\epsilon,\delta)$, and window size $W$.} 
\Ensure{Updated matrices ${\widetilde K(1), \cdots,  \widetilde K(\ell+1)}$ and timestamps $t_{1}, \cdots, t_\ell$.}

\State {\bf if} {$t_{2} < t- W+1$}
\label{step:stage1start}
    \State  \hspace{5mm}{\bf Set} $t_{j} = t_{j+1},\widetilde K(j) := \widetilde K(j+1) $ for $1\leq j \leq s-1$ 
    \Comment{\small{Delete the expired timestamp.}}
    \label{step:expiredmain}
\State {\bf end}
\label{step:t_2expiredmain}

\State {\bf Set} $t_{\ell+1} = t$, sample 
    $$R \sim \mathsf{Wis}_d(\tau,\I_d),~\text{where}~\tau \geq \left\lfloor d + \frac{14}{\epsilon^2} \log(1/\delta) \right\rfloor.$$ 

\State {\bf Define} $\widetilde K(\ell+1) = a_t^\top a_t + R.$
\label{step:newcheckpoint}

\State {\bf Include} $\dpriv \leftarrow \dpriv \cup ( \widetilde K(\ell+1) ,t)$. 

\State {\bf for} {$i= 2,\ldots, \ell$ }
\label{step:updatecheckpoint}
\State  \hspace{5mm} {\bf Compute} $ \widetilde K(i) \leftarrow \widetilde K(i) + a_t^\top a_t.$ 
\Comment{\small{Update the matrices.}}
\State {\bf end}
\label{step:stage1end}

\State {\bf Find} 
$j := \min \set{p : K(p) \not\succeq K(\ell)  }$. 

\State {\bf Delete} $K(p), \cdots, K(\ell-1)$  
\State {\bf Update} $ K(p) = K(\ell), \ell =p.$
\Comment{\small{Mantain PSD ordering.}}

\State{\bf Return} $\dpriv := \set{(\widetilde K(i), t_i)}_{i=1}^\ell$.
\end{algorithmic}
\end{algorithm}

We use both these subroutines in our main algorithm, {\scshape Sliding-Priv}. {\scshape Sliding-Priv} receives a stream of rows and call these two subroutines on every new update. Equipped with Lemma~\ref{lem:smoothpsd}, we show that  {\scshape Sliding-Priv}, described in Algorithm~\ref{fig:slidingpriv}, provides the following guarantee.

\begin{theorem}
[Private spectral approximation under sliding window]
\label{thm:privsliding}
Given the privacy parameter $\epsilon$, window size $W$, approximation parameter $\beta$, let $S = (a_t)_{t>0}$ be the stream such that $a_t \in \R^d$. Further define $A_W$ to be the matrix formed at time $T$ by the last $W$ updates.
Then we have the following:
\begin{enumerate}
    \item {\scshape Sliding-Priv}$(\Omega;(\epsilon,\delta);W)$, described in Algorithm~\ref{fig:slidingpriv}, is $(\epsilon,\delta)$-differential private.
    \item $\widetilde C \leftarrow ${\scshape Sliding-Priv}$(\Omega;(\epsilon,\delta);W)$ satisfies the following:
    \[
     \pr \sparen{\paren{A_{W}^\top A_W - (c \tau \log \tau) \I_d } \preceq \widetilde C  \preceq \paren{\frac{1}{(1-\eta)} A_W^\top A_W  + (C \tau \log \tau) \I_d}} \geq 1 - \frac{1}{\poly(d)} 
    \]
    for  constants $c,C>1$ and $\tau:= d + \frac{14 \log (4/\delta)}{\epsilon^2}$.
    \item The space required by {\scshape Sliding-Priv} is $O\left(\frac{d^3}{\eta} \log W\right)$.
\end{enumerate}
\end{theorem}

\begin{algorithm}[b]
\caption{{\scshape Sliding-Priv}$(\Omega;(\epsilon,\delta);W)$}
\label{fig:slidingpriv}
\begin{algorithmic}[1]
\Require{A stream, $\Omega$, of row $\set{a_t}$, privacy parameters $(\epsilon,\delta)$, and window size $W$.} 
\Ensure{A positive semidefinite matrix $ {\widetilde C}$ at the end of the stream.}

\State {\bf Initialize} $\dpriv$ to be an empty set, $r=d$. 

\State {\bf while} {stream $S$ has not ended}
    \State  \hspace{5mm} {\bf Include} new row, $\dpriv \leftarrow$ {\scshape Priv-Initialize} $(\dpriv;a_t; t; (\epsilon,\delta);  W; r)$. \Comment{Algorithm~\ref{fig:slidingprivinitialize}}
    
    \State  \hspace{5mm} {\bf Update} the data structure, $\dpriv \leftarrow${\scshape Update}$(\dpriv)$.
    \Comment{Algorithm~\ref{fig:slidingprivupdate}}
    
\State {\bf end}

\State {\bf Let} $\dpriv = \set{( \widetilde K(1),t_1), \ldots, (\widetilde K(\ell),t_\ell)}$  for some $\ell$.

\State {\bf Output} $\widetilde C = \widetilde K(1)$.

\end{algorithmic}
\end{algorithm}

\begin{proof}
[Proof of Theorem~\ref{thm:privsliding}]
Consider an index $1 \leq i \leq \ell$ and the time epoch $t$ when the stream of rows are different resulting in neighboring matrices $A_{[t_i,t]}$ and $A'_{[t_i,t]}$, we have 
\[
A_{[t_i,t]}^\top A_{[t_i,t]} - (A'_{[t_i,t]})^\top A_{[t_i,t]}' = u^\top u,
\]
where $u$ is a unit row vector. That is $A_{[t_i,t]}^\top A_{[t_i,t]} - (A'_{[t_i,t]})^\top A_{[t_i,t]}'$ is a rank-$1$ matrix. Now $W \sim \mathsf{Wis}_d(\tau, \I_d)$. 

Let $\mathcal P$ denote the output distribution of our mechanism when run on the input matrix $A_{[t_i,t]}$ and similarly let $\mathcal Q$ denote the output of our algorithm on input matrix $A'_{[t_i,t]}$. Both distribution are supported on $\mathbf S:=\R^{d \times d}$ matrices. For $M \in \mathbf S$, consider the privacy loss function 
\[
L(M) := \log \paren{ \frac{\mathcal P(M)}{\mathcal Q(M)} }.
\]
When $T < t$, the output distribution of $\mathcal P$ and $\mathcal Q$ are identical, i.e., $L(M)=0$. When $T=t$, the privacy proof follows by the choice of $\tau$ and  Theorem~\ref{thm:wishartprivacy}. That is, 
\[
\pr [L(M) \leq \epsilon] \geq 1 - \delta.
\]
For any time $T \geq t$, we have differential privacy because of the post-processing property (Lemma~\ref{lem:post}). 

For the space bound, note that the number of checkpoints stored by $\dpriv$ is $O\left(\frac{d}{\eta}\log W\right)$. This is because there are exactly $d$ singular values and the matrix has polynomially bounded spectrum. Since at each checkpoints defined by $t_i$ for $i \geq 1$ stores an $d \times d$ matrix, the total space used by the data structure $\dpriv$, and hence the algorithm {\scshape Sliding-Priv}, is $O\left(\frac{d^3}{\eta} \log W\right)$. 

Now we turn our attention to the accuracy guarantee. We start by noting that the output of {\scshape Sliding-Priv}$(S; (\epsilon, \delta); W)$ is $\widetilde K(1)$, the first positive semidefinite matrix in the data structure $\dpriv$.

Let $K(1)$ and $K(2)$ denote the covariance matrix formed between time epochs $[t_1,T]$ and $[t_2,T]$, respectively.Since the window is sandwiched between the first and second timestamp, and we preserve the positive semidefinite ordering,  we have
\begin{align}
    K(2) \preceq A_W^\top A_W \preceq K(1).
    \label{eq:windowapprox}
\end{align} 

\noindent Let $R(1)$ and $R(2)$ be matrices sampled from the Wishart distribution such that
\[
\widetilde K(1):= K(1) + R(1) \quad \text{and} \quad
\widetilde K(2):= K(2) + R(2).
\]

\noindent Note that $\dpriv$ stores the set $\left\{\widetilde K(1), \widetilde K(2), \cdots, \widetilde K(\ell)\right\}$. From the spectral histogram property of the matrices in $\dpriv$, we have the following relation between $\widetilde K(1)$ and $\widetilde K(2)$: 
\begin{align}
    (1-\eta)  \widetilde A^\top \widetilde A = (1-\eta)  \widetilde K(1) \preceq  \widetilde K(2).
    \label{eq:spectral}
\end{align}
Define 
\[
\sigma := \tau \log (\tau) = \paren{d + \frac{\log(1/\delta)}{\epsilon^2}} \log \paren{d + \frac{\log(1/\delta)}{\epsilon^2}}.
\]

Using the standard result on the eigenvalue bounds of matrices sampled from Wishart distribution~\cite{johnstone2001distribution, soshnikov2002note}, we have that  $\lambda_1( R(1)) \leq c \sigma $ and $\lambda_1( R(2)) \leq c \sigma$ for some constant $c>1$ with probability $1 - \frac{1}{\poly(d)}$. Also, since $R(1)$ and $R(2)$ are sampled from a Wishart distribution, they are positive semidefinite. Therefore, we have 
\begin{align}
\begin{split}
    \pr \sparen{ {K(1) - c_1 \sigma \I_d } \preceq \widetilde K(1)
\preceq 
 { K(1) + c_1 \sigma \I_d }} \geq 1 - \frac{1}{\poly(d)}, \\
\pr \sparen{ {K(2) - c_2  \sigma \I_d } \preceq \widetilde K(2)
\preceq 
 { K(2) + c_2 \sigma \I_d }} \geq 1 - \frac{1}{\poly(d)},
\end{split} 
\label{eq:approxG_1}
\end{align}
where $K(1)$ is the underlying covariance matrix formed during the time epochs $[t_1,t]$, $K(2)$ is the underlying covariance matrix formed during the time epochs $[t_2,t]$, and $c>0$ is a constant.  

We now condition on the event that Equation~(\ref{eq:approxG_1}) holds for the rest of the proof. Let $\sigma := \tau \log \tau$. 
Using Equations~(\ref{eq:windowapprox}),~(\ref{eq:spectral}) and~(\ref{eq:approxG_1}), we  arrive at  
\begin{align}
\begin{split}
 (1-\eta) \paren{ K(1) - c_1 \sigma \I_d} &\preceq (1-\eta)\widetilde K(1)  \\
    &\preceq \widetilde K(2) \preceq K(2) + c_2 \sigma \I_d \\
    & \preceq  A_W^\top A_W + c_2 \sigma \I_d 
\end{split}
\label{eq:approxG_1G_2L_W}    
\end{align}

\noindent Rearranging the terms in Equation~(\ref{eq:approxG_1G_2L_W}) gives us 
\begin{align}
    K(1) \preceq \frac{1}{(1-\eta)} A_W^\top A_W  + c_3 \sigma \I_d, 
    \label{eq:approxG_1L_W}
\end{align}
where $c_3 = c_1 + \frac{c_2}{(1-\eta)}$ is a constant. Using Equation~(\ref{eq:approxG_1L_W}) in the right side positive semidefinite inequality of Equation~(\ref{eq:approxG_1}), we have 
\begin{align}
   \paren{K(1) - c_1 \sigma \I_d } \preceq \widetilde K(1)
\preceq 
\frac{1}{(1-\eta)} A_W^\top A_W  + c_3 \sigma \I_d.
\label{eq:approxG_1tildeG_1L_W}  
\end{align}
Using the left hand semidefinite inequality of Equation~(\ref{eq:windowapprox}) in Equation~(\ref{eq:approxG_1tildeG_1L_W}), we get 
\begin{align}
   \paren{A_{W}^\top A_W - c_1 \sigma \I_d } \preceq \widetilde K(1) 
\preceq 
\frac{1}{(1-\eta)} A_W^\top A_W  + c_3 \sigma \I_d.
\label{eq:approxfinal}  
\end{align}
Since $\widetilde C = \widetilde K(1)$ by the output of the algorithm, we have the desired bound. 
\end{proof}

\subsection{Application of Algorithm~\ref{fig:slidingpriv}: sub-optimal algorithms for private matrix analysis}
\label{sec:applicationwishart}
Theorem~\ref{thm:privsliding} gives the guarantee that Algorithm~\ref{fig:slidingpriv} outputs a matrix that approximates the spectrum of $A_W^\top A_W$ up to a small additive error in the spectrum. This in particular means that Algorithm~\ref{fig:slidingpriv} can be used to solve many matrix analysis problems; however, the accuracy guarantees are sub-optimal in many cases. For every problem discussed in this appendix, we give a pointer to the improved bounds. 

As a warm up, we consider directional variance queries. Theorem~\ref{thm:covariance} is true for any $d$-dimensional unit vector $x \in \R^d$. 

\subsection{Directional Variance Queries}
\label{sec:covariance}
The directional variance queries has the following form: the analyst gives a unit-length vector $x \in \R^d$ and wish to know the variance of $A_W$ along $x$. Using Theorem~\ref{thm:wishartprivacy}, we have the following result. 
\begin{theorem}
[Directional variance queries]
\label{thm:covariance}
Let $A_W$ be the matrix formed by last $W$ updates as defined in equation~(\ref{eq:sliding}), $\eta$ be the given approximation parameter, $(\epsilon,\delta)$ be the privacy parameter. Then there is an efficient $(\epsilon,\delta)$-differentially private algorithm that outputs a matrix $C$ such that for any unit vector $x \in \R^d$, we have
\[
 x^\top A_W^\top A_W x - c_1 \tau \log(\tau) \leq x^\top C x \leq \frac{1}{(1-\eta)} x^\top A_W^\top A_W x + c_3 \tau \log(\tau),
\]
where $\tau = d + \frac{14}{\epsilon^2}\log(4/\delta)$.
\end{theorem}
\begin{proof}
The proof follows immediately from Equation~(\ref{eq:approxfinal}) and the fact that $\brak{x,x} = \norm{x}_2 =1$.
\end{proof}

A special case of directional variance queries is when $a_t$ is the edges of a weighted graph, $d=n$, and the query is of form $\set{0,1}^n$. Such a query is known as cut queries. Using Theorem~\ref{thm:covariance}, we have the following result for answering cut queries.
\begin{corollary}
[Cut queries]
\label{cor:cut}
Let $\cG_W$ be the graph formed by last $W$ updates as defined in equation~(\ref{eq:sliding}). There is an efficient $(\epsilon,\delta)$-differentially private algorithm that outputs a matrix $C$ such that for any cut query $S \subseteq [n]$, we have
\[
 \Phi_S(\cG_W) - \frac{c |S| \sqrt{\tau \log \tau}}{\epsilon} \leq \mathsf{Out}_S \leq \frac{1}{(1-\eta)} \Phi_S(\cG_W) + \frac{c |S| \sqrt{\tau \log \tau}}{\epsilon},
\]
where  $\tau = n + \frac{14}{\epsilon^2}\log(4/\delta)$ and $\mathsf{Out}_S = \sqrt{e_S^\top C e_S}$ for $e_S:= \sum_{i \in S} \bar e_i $.
\end{corollary}
\begin{proof}
The matrix $C$ is the same as that generated in Theorem~\ref{thm:covariance} and on a query, $S \subseteq [n]$, we output 
\[ \mathsf{Out}_S = \sqrt{e_S^\top C e_S}, \text{ where } e_S:= \sum_{i \in S} \bar e_i .\]  
Using Theorem~\ref{thm:covariance} and the bound on the eigenvalue of Wishart matrices~\cite{johnstone2001distribution}, we have the result.
\end{proof}

Note that this is better than the result of~\citet{blocki2012johnson}. They achieved a multiplicative error of $(1 \pm \eta)$ and additive bound of $O\paren{\frac{|S|\sqrt{n \log(1/\delta)}}{\epsilon} \log(1/\delta) }$ for a cut query $S$. 

In practice, it is not always feasible to ask all possible queries and only a polynomial number of queries. In Section~\ref{sec:applicationJL}, we showed  (Theorem~\ref{thm:covariancelimited}) that we can get a better bound if we have an a priori bound $q$ on the number of queries an analyst can make.

\subsection{Principal Component Analysis}
Since Theorem~\ref{thm:privsliding} preserves the spectrum of the covariance matrix, it can be used for a variety of tasks involving spectrum. In particular, we can use it to compute the principal component of the matrix streamed in the window. Let $\Pi$ be the set of all rank-$k$ orthonormal projection matrices, i.e., every matrix $P \in \Pi$ has rank $k$ and satisfy $P^2 = P$ and $P = P^\top$. 
 
 Now consider the following algorithm:
 \begin{center}
     \underline{\scshape Sliding-PCA}
 \end{center}
\begin{enumerate}
    \item Compute $\widetilde A^\top \widetilde A \leftarrow${\scshape Sliding-Priv}$(S;(\epsilon,\delta);W)$.
    \item Output \label{step:rankconstrainedPCA}
        \[
        \widetilde X = \argmin_{P \in \Pi} \set{ \tr {(\I_d - P)^\top \widetilde A^\top \widetilde A (\I_d -P)}}.
        \]
\end{enumerate}

We note that the rank constrained problem in step~\ref{step:rankconstrainedPCA} can be solved efficiently using the result of~\citet{upadhyay2018price} and present a self-contained proof in Section~\ref{sec:rankconstrained}. We have the following result for the {\scshape Sliding-PCA} algorithm.

\begin{theorem}
\label{thm:pca}
Given privacy parameters $(\epsilon,\delta)$ and approximation parameter $\eta \in (0,1/2)$, let $A_W$ be the matrix formed by the last $W$ updates  as defined in equation~(\ref{eq:sliding}) and $\Pi$ be the set of all rank-$k$ orthonormal projection matrices. Then {\scshape Sliding-PCA} is an efficient $(\epsilon, \delta)$-differentially private algorithm that outputs a rank-$k$ projection matrix $\widetilde X \in \R^{d \times d}$ such that 
\[
\pr \sparen{\norm{A_W(\I_d - \widetilde X)}_F^2 \leq (1+2\eta) \min_{P \in \Pi} \norm{A_W(\I_d - P)}_F^2 + O\paren{ d\tau \log(\tau) }} \geq 1 - \frac{1}{\poly(d)},
\]
where $\tau := \paren{d + \frac{14}{\epsilon^2} \log(4/\delta)}.$
\end{theorem}
\begin{proof}
Let 
\begin{align}
\widehat X := \argmin_{X \in \Pi} \norm{A_W(\I_d - X)}_F^2 
\qquad \text{and} \qquad
\widetilde X := \argmin_{X \in \Pi} \norm{\widetilde A(\I_d - X)}_F^2.
\label{eq:hatXtildeXmin}    
\end{align}

\noindent Then from the optimality of $\widetilde X$ and the fact that $\norm{Y}_F^2 = \tr{Y^\top Y}$ for any matrix $Y$, we have
\begin{align*}
    \norm{\widetilde A(\I_d - \widetilde X)}_F^2 & \leq \norm{\widetilde A(\I_d - \widehat X)}_F^2
    = \tr {(\I_d - \widehat X)^\top \widetilde A^\top \widetilde A (\I_d - \widehat X)} .
\end{align*}
{Using the right hand side semidefinite inequality in Equation~(\ref{eq:approxfinal}), we have}
\begin{align}
    \norm{\widetilde A(\I_d - \widetilde X)}_F^2  &\leq \frac{1}{(1-\eta)}\tr {(\I_d - \widehat X)^\top A_W^\top  A_W (\I_d - \widehat X)} + c  (\I_d - \widehat X)^\top (\I_d - \widehat X) \tau \log(\tau) \nonumber  \\
    &\leq \frac{1}{(1-\eta)}\norm{A_W(\I_d - \widehat X)}_F^2 + cd \paren{d + \frac{14}{\epsilon^2} \log(4/\delta)} \log \paren{d + \frac{14}{\epsilon^2} \log(4/\delta)} \nonumber  \\
    &= \frac{1}{(1-\eta)}\min_{X \in \Pi} \norm{A_W(\I_d - X)}_F^2 + 
    c d \paren{d + \frac{14}{\epsilon^2} \log(4/\delta)} \log \paren{d + \frac{14}{\epsilon^2} \log(4/\delta)}.
    \label{eq:tildeAupper}
\end{align}
where the first inequality follows from the fact that $(\I_d-\widehat X)$ is a rank $d-k$ projection matrix and second equality follows from equation~(\ref{eq:hatXtildeXmin}).

Now using the left hand side inequality of Equation~(\ref{eq:approxfinal}) and the fact that $(\I_d-\widehat X)$ is a rank $d-k$ projection matrix, we have 
\begin{align}
    \norm{\widetilde A(\I_d - \widetilde X)}_F^2 
    &= \tr {(\I_d - \widetilde X)^\top \widetilde A^\top \widetilde A (\I_d - \widetilde X)}\nonumber \\
    &\geq \tr {(\I_d - \widetilde X)^\top  A_W^\top  A_W (\I_d - \widetilde X)} - c  (\I_d - \widetilde X)^\top (\I_d - \widetilde X) \tau \log(\tau) \nonumber \\
    &\geq  \norm{A_W(\I_d - \widetilde X)}_F^2 - cd \paren{d + \frac{14}{\epsilon^2} \log(4/\delta)} \log \paren{d + \frac{14}{\epsilon^2} \log(4/\delta)}. 
    \label{eq:tildeAlower}
\end{align}

\noindent Combining equations~(\ref{eq:tildeAupper}) and~(\ref{eq:tildeAlower}), we have Theorem~\ref{thm:pca}.
\end{proof}


\section{Extension to continual release}
\label{sec:continual}
Until now, we consider only one-shot algorithm, that is,  an algorithm to compute spectral approximation with  additive error of  $O \paren{ \frac{cr \log^2(1/\delta)}{\epsilon^2} } \I_d$, but the output is produced just once. If we naively use this algorithm to publish a matrix continually over the entire window, it would lead to a total accuracy loss of  $O \paren{ \frac{cr W \log^2(1/\delta)}{\epsilon^2} } \I_d$. In this section, we show an algorithm that computes spectral approximation with small additive error over the entire window, i.e., $o(W\tau)$. 

The continual release model was proposed by Dwork et al.~\cite{dwork2010differentially}. In contrast to our setting, continual release model consider the entire data useful and does not put any space constraints. We provide two different protocols, in both of which we consider accuracy for only the update that came during the current window. 

The first approach uses the same binary tree method introduced by Bentley and Saxe~\cite{bentley1980decomposable} and used in Dwork et al.~\cite{dwork2010differentially} and Chan et al.~\cite{chan2011private}, and in the sliding window model by Bolot et al.~\cite{bolot2013private} and Upadhyay~\cite{upadhyay2019sublinear}. However, we depart from their technique in the sense that we only build the binary tree. Let $a_{T-W+1}, \cdots, a_T$ be the updates at any time $T$.
In particular, we construct a binary tree as follows:
\begin{enumerate}
    \item Every leaves consists of a single update privatized using Step~\ref{step:newcheckpoint}.
    \item For every other node, $\mathsf n$, other than the leaf nodes, let $C$ be the set of updates on the leaves of the subtree of $\mathsf n$. Then we first compute 
    \[
    S_{\mathsf n} = \sum_{a_i \in C} a_i^\top a_i
    \]
    Then we store $\widetilde S_{\mathsf n}$ on the node $\mathsf n$, where $\widetilde S_{\mathsf n}$ is formed using the privitization step (Step~\ref{step:newcheckpoint}) on $S_{\mathsf n}$.
\end{enumerate}

This construction mimics the construction of Dwork et al.~\cite{dwork2010differentially} and hence using their analysis, we get the following result:
\begin{theorem}
[Private spectral approximation under sliding window]
\label{thm:privslidingcontinual}
Given the privacy parameter $\epsilon$, window size $W$, approximation parameter $\beta$, let $\Omega = (a_t)_{t>0}$ be the stream such that $a_t \in \R^d$. For every $t>0$, define $A_W(t)$ to be the matrix formed at time $t$ by the last $W$ updates.
Then we have the following:
\begin{enumerate}
    \item {\scshape Sliding-Priv}$(\Omega;(\epsilon,\delta);W)$, described in Algorithm~\ref{fig:slidingpriv}, is $(\epsilon,\delta)$-differential private.
    \item $\widetilde C \leftarrow ${\scshape Sliding-Priv}$(\Omega;(\epsilon,\delta);W)$ satisfies the following:
    \[
     \pr \sparen{{A_{W}^\top A_W - (c \tau \log \tau) \I_d } \preceq \widetilde C  \preceq { A_W^\top A_W  + (C \tau \log \tau) \I_d}} \geq 1 - \frac{1}{\poly(d)} 
    \]
    for  constants $c,C>1$ and $\tau:= \paren{d + \frac{14 \log(1/\delta)}{\epsilon^2}}  \log^{3/2} \paren{W} $.
    \item The space required by {\scshape Sliding-Priv} is $O\left(\frac{d^2W}{\eta} \log W\right)$.
\end{enumerate}
\end{theorem}
Note that this result uses $\eta$-spectrogram property. 


\subsection{Making space requirement sublinear in window-size at the cost of accuracy loss}

We now improve this bound by incurring an accuracy loss that scales only logarithmic in the window size instead of linear. For this, we borrow the idea of Bentley and Saxe~\cite{bentley1980decomposable} to move from one-shot algorithms to continually release algorithm. This technique was also used in Dwork et al.~\cite{dwork2010differentially} and subsequently improved in Chan et al.~\cite{chan2011private} and Bolot et al.~\cite{bolot2013private}. The idea is to build binary tree with leaves being the matrix at the checkpoint. 
For this, we fix some notation:
\begin{enumerate}
    \item $\widetilde{\mathfrak B}$ be the binary tree formed by the leaves $\widetilde K(1), \cdots ,\widetilde K(\ell)$.
    \item $\widetilde{\mathfrak B}_{\mathsf n}$ be the subtree of the internal node, $\mathsf{n}$, of the tree $\widetilde{\mathfrak B}$. 
    \item $\widetilde{\mathfrak L}_{\mathsf n}$ be the leaves of $\widetilde {\mathfrak B}$ in the subtree $\widetilde{\mathfrak B}_{\mathsf n}$; i.e., a subset of the graphs $\widetilde K(1), \cdots ,\widetilde K(\ell)$.
\end{enumerate}

We divide our window in to $\sqrt{W}$ sub-windows, each of size $\sqrt{W}$. We run an instantiation of our algorithm for each of these subwindows. Let these subwindows terminates at timestamps $T_1, T_2, \cdots, T_{\sqrt W} =T$. For $j$-th subwindow that terminates at time $T_j$, we also augment our data structure 
$\dpriv$ for each of these windows to contains the following:
\begin{enumerate}
    \item A set of covariance matrix for every timestamps stored in the data structure in Section~\ref{sec:slidingwishart}. That is, for timestamps, $t_1, \cdots, t_\ell$, apart from the privatized covariance matrix, $\widetilde K(1), \cdots, \widetilde K(\ell)$  we also store $K(i)$ such that 
    \[
    K(i) = \sum_{t=t_i}^{T_j} a_t^\top a_t \quad \text{for all}~1 \leq i \leq \ell.
    \]
    
    \item A binary tree formed using an algorithm {\scshape Binary-Tree} that uses covariance matrices $K(1), \cdots, K(\ell)$ and $\widetilde K(1), \cdots, \widetilde K(\ell)$. {\scshape Binary-Tree} operates as follows:
    \begin{enumerate}
        \item The leaves of the tree are $\widetilde K(1), \cdots, \widetilde K(\ell)$.
        \item For every internal node, ${\mathsf n}$, let $\mathfrak L_{\mathsf n}$ be the covariance matrix from the set $\set{K(1), \cdots, K(\ell)}$ corresponding to the covariance matrices in the set $\widetilde{\mathfrak L}_{\mathsf n}$. Then the covariance matrix stored in the node ${\mathsf n}$ is the privatization of the following covariance matrix:
        \begin{align*}
            K_{\mathsf n} = \sum_{\widetilde K(i) \in \widetilde{\mathfrak L}_n} K(i),
        \end{align*}
        where the privatization is done as in Step~\ref{step:newcheckpoint}.
    \end{enumerate}
    \item Delete all the internal nodes whose leaves contains covariance matrix is formed before time $t_1$.
\end{enumerate}

Since the number of checkpoints is $\ell = O\paren{\frac{n}{\rho} \log W}$, combining Theorem~\ref{thm:privsliding} with that of Dwork et al.~\cite{dwork2010differentially}, we have the following theorem:

\begin{theorem}
[Private spectral approximation under sliding window]
\label{thm:privslidingcontinualsmall}
Given the privacy parameter $\epsilon$, window size $W$, approximation parameter $\beta$, let $\Omega = (a_t)_{t>0}$ be the stream such that $a_t \in \R^d$. For every $t>0$, define $A_W(t)$ to be the matrix formed at time $t$ by the last $W$ updates.
Then we have the following:
\begin{enumerate}
    \item {\scshape Sliding-Priv}$(\Omega;(\epsilon,\delta);W)$, described in Algorithm~\ref{fig:slidingpriv}, is $(\epsilon,\delta)$-differential private.
    \item $\widetilde C \leftarrow ${\scshape Sliding-Priv}$(\Omega;(\epsilon,\delta);W)$ satisfies the following:
    \[
     \pr \sparen{\paren{A_{W}^\top A_W - (c \tau \log \tau) \I_d } \preceq \widetilde C  \preceq \paren{\frac{A_W^\top A_W}{(1-\eta)}   + (C \tau \log \tau) \I_d}} \geq 1 - \frac{1}{\poly(d)} 
    \]
    for  constants $c,C>1$ and $\tau:= \paren{d + \frac{14 \log(1/\delta)}{\epsilon^2}} W^{3/4}$.
    \item The space required by {\scshape Sliding-Priv} is $O\left(\frac{d^3\sqrt{W}}{\eta} \log W\right)$.
\end{enumerate}
\end{theorem}


\section{Lower Bounds for Low-rank Approximation}
\label{sec:lower}
\label{app:lower}
This section is devoted to proving a lower bound on the space requirement for low-rank factorization with non-trivial additive error. It is well known that no private algorithm (not necessarily differentially private) incurs an additive error $o(\sqrt{kd})$~\cite{hardt2012beating} due to linear reconstruction attack. 
On the other hand, the only known space lower bound of \citet{upadhyay2018price} holds for streaming data where the entire historic data is considered important. While the entries can be streamed in an arbitrary order, this paper considers the case when one row is streamed at a time. Hence, there might be a possibility to construct an improved space algorithm for the special case of streaming we consider. However, we show below that for any non-trivial values of $\tau$, this is not the case. 

We first note that the technique developed by Bar-Yossef~\citet{bar2002complexity} can be used to give lower bounds on the number of rows to be sampled by any sampling-based algorithm for low-rank matrix approximation. However, space lower bounds, in general, is a harder problem as one can use methods other than row sampling. For example, Bar-Yossef~\citet{bar2002complexity} showed that any sampling-based algorithm for computing Euclidean norm of a stream of length $W$ requires $\Omega(W)$ samples, while Upadhyay~\citet{upadhyay2019sublinear} gave a privacy-preserving sliding window algorithm using $O\paren{\frac{\sqrt{W}\log^2 W}{\eta^2}}$ bits. Our lower bounds come from reduction from the two-party communication complexity of augmented indexing, $\mathsf{AIND}$ problem~\cite{miltersen1995data}.

\begin{theorem} \label{thm:lower}
Let $n,d,k \in \mathbb N$ and $\eta >0$. Then the space used by any randomized single-pass algorithm for low-rank approximation in the sliding window model is at least $\Omega(Wk \log(W)/\eta)$.
\end{theorem}
\begin{proof}
For a matrix $ A $ and set of indices $C$, we use the notation $ A (C)$ to denote the submatrix formed by the columns indexed by $C$. We use the standard extension of the proof of \citet{upadhyay2018price} for the sliding window model. The idea is basically for Alice to generate a stream with heavier weights on the more recent rows. Then Bob simply discards the stream not in the last $W$ updates and use the rest of the state to compute the value of $x_{\mathsf{ind}}$ as in the case of \citet{upadhyay2018price}. Let $\ell = \frac{\log W}{\eta}$.  Suppose $n \geq d$ and let $a=\frac{k\ell}{20 \eta}$. Without loss of generality, we can assume that $a$ is at most $d/2$. We assume Alice has  a string $ x \in \set{-1,+1}^{(W-a)a}$ and Bob has an index ${\mathsf{ind}} \in [(W-a)a]$. The idea is  to define the matrix $ A $ with high Frobenius norm. The matrix $ A $ is the summation of the matrix $\widetilde{ A }$ constructed by Alice and $\bar{ A }$ constructed by Bob. We first define how Alice and Bob construct the instant $ A = \widetilde{ A }+\bar{ A }$.

 Alice constructs its matrix $\widetilde{ A }$ as follows. 
 \begin{enumerate}
     \item Alice partitions the set $\set{1,\cdots, a}$ in to $\ell $ disjoint sets $I_1, \cdots, I_\ell$ such that $$I_i:= \set{(i-1)a/\ell +1, \cdots ia/\ell}.$$
     
     \item Let $ M  \paren{I_i}$ be an $(W-a) \times \frac{a}{\ell}$ matrix for all $1 \leq i \leq \ell$.
     
     \item Alice forms a bijection between entries of $ x$ and the entries of $ M $ in the following manner. Every entry of $ M \paren{I_i}$ is defined by a unique bit of $ x$, i.e.,  
     $$ M \paren{I_i}_{j,k} = (-1)^{ x_{p}} (10)^i, \quad p=\frac{(i-1)(W-a )a}{\ell} + (k-1)(n-a ) + j.$$
     
     \item The matrix $\widetilde{ A }$ is now defined as follows.
    \[ \widetilde{ A } = \begin{pmatrix}  0^{a \times a} &  0^{a \times (d-a)} \\  M &  0^{(n-a) \times (d-a)} \end{pmatrix}, \]
    where $M = \begin{pmatrix}  M _{I_1} & \cdots &  M _{I_\ell} \end{pmatrix}$.
 \end{enumerate}

Suppose Bob is given an index ${\mathsf{ind}} \in [(W-a)a]$ such that $ x_{\mathsf{ind}}$ corresponds to the sub-matrix $ M \paren{I_\theta}$ for some $1 \leq \theta \leq \ell$. Then we can assume that Bob also knows every entry in the sub-matrix $ M \paren{I_{\theta'}}$ for $\theta' >\theta$. The idea is that Bob inserts a scaled identity matrix in the stream, where the scaling parameter $\gamma$ is large enough to make sure that most of the error of any randomized algorithm is due to other columns of $ A $. As we shall see later, we set the value of $\gamma$ as a large polynomial in the approximation error of the algorithm.  Bob forms his matrix as follows:
\begin{enumerate}
    \item Bob forms a second level partition of the columns of $ M \paren{I_\theta}$ into equal size groups
    $G_1, \cdots, G_{a/k \ell}.$ 
    There exists a unique $r$ such that $ x_{\mathsf{ind}}$ maps to an entry in the sub-matrix formed by columns indexed by one of the second level partition $G_r$.
    
    \item Let  $C = \set{c, c+1, \cdots, c+k-1}$ be the columns corresponding to the group of $I_\theta$ in which ${\mathsf{ind}}$ is present.
    
    \item Bob expires the stream of Alice except for the current window and stream in  a matrix $\bar{ A }$ which is an all-zero matrix, except for entries $\bar{ A }_{c+i,c+i} = \gamma$ for $0 \leq i \leq k-1$ and $\gamma$ to be chosen later.
\end{enumerate}

Let $\cA$ be the algorithm that computes low-rank approximation under the turnstile model. Alice feeds its matrix $\widetilde{ A }$ to $\cA$ in the turnstile manner and send the state of the algorithm by the end of her feed to Bob. Bob  uses the state received by Alice and feed the algorithm $\cA$ with its own matrix $\bar{ A }$ in a turnstile manner. Therefore, the  algorithm $\cA$ gets as input a matrix $ A  = \widetilde{ A } + \bar{ A }$ and it is required to output a rank-$k$ matrix $ B $ with additive error $\tau=O(W+d)$. We will show that any such   output allows us to solve $\mathsf{AIND}$. Denote by ${ A }(C)$ the sub-matrix formed by the columns $C:=\set{c, c+1, \cdots, c+k-1}$.

Let us first understand the properties of the constructed matrix $ A $. To compute the Frobenius norm of this matrix, we need to consider two cases: the case for sub-matrices in which ${\mathsf{ind}}$ belongs, i.e, $ M \paren{I_r}$, and the rest of the matrix. For the sub-matrix corresponding to the columns indexed by ${C}$, the columns of $ A \paren{I_\theta}$ have Euclidean length $(\gamma^2 + (n-a)100^\theta)^{1/2}$. 
For $\theta' <\theta$, every columns have Euclidean norm  $(a (n-a ))^{1/2}10^{\theta'}$. 
Therefore, we have the following:
\begin{align*}
	\| { A } - [{ A }]_k \|_F^2 &\leq \frac{((a-k)(W-a )100^\theta}{\ell} + \sum_{\theta' <\theta} \frac{a (W-a )100^{\theta'}}{\ell} \\
		& \leq \frac{((a-k)(W-a )100^\theta}{\ell} +  \frac{a (W-a )100^\theta}{99 \ell} \leq  2 \cdot (100)^\theta Wd/\ell = \tau
\end{align*}

In order to solve low-rank approximation, the algorithm needs to output a matrix $ B $ of rank at most $k$ such that, with probability $5/6$ over its random coins, 
\begin{align*}
	\| { A } -  B  \|_F^2 &\leq \sparen{ (1+\eta) \sqrt{\tau} + \tau }^2 \leq 2(1+\eta) \tau + 2 \tau^2 \\
	& \leq 2\tau + 100^\theta k(W-a ) \paren{\frac{1}{10} + \frac{1}{99}}  + 2\tau^2\\
		&\leq 4 \cdot (100)^\theta Wd/\ell  + \frac{100^\theta k(n-a ) }{5}  + 2\tau^2
\end{align*}

Let $\Psi:=4 \cdot (100)^\theta Wd/\ell + 100^\theta k(W-a ) \paren{\frac{1}{10} + \frac{1}{99}}+ 2\tau^2$. 
The proof idea is now to show the following:
\begin{description}
	\item [Step 1.] Columns of $ B $ corresponding to index set in $C$ are linearly independent. 
	\item [Step 2.] Bound the error incurred by $\| { A } -  B  \|_F$ in terms of the columns indexed by $G_r$.
\end{description}

The idea is to show that most of the error is due to the other columns in $ B $; and therefore, sign in the submatrix $ A (C)$ agrees with that of the signs of those  in the submatrix $ B (C)$. This allows Bob to solve the $\mathsf{AIND}$ problem as Bob can just output the sign of the corresponding position.
Let $$R:=\set{ra/k+1, \cdots, (r+1)a/k} \text{ and } C:=\set{c,\cdots, c+k-1}.$$ Let $ Y $ be the submatrix of $ B $ formed by the rows indexed by $R$ and columns indexed by $C$.

\paragraph{Columns of $B$ are linearly independent.}
The following lemma proves that when $\gamma$ is large enough, then the columns of $ B $ corresponding to index set $C$ are linearly independent. 
\begin{lemma} \label{lem:independent}
Let $ B (C):= [\begin{matrix}  B _{:c} & \cdots  B _{:c+k-1} \end{matrix}]$ be the columns corresponding to the sub-matrix formed by columns $c, \cdots, c+k-1$ of $ B $. If $\gamma \geq 2\Psi^2$, then the columns of $ B (C)$ spans the column space of $[ A ]_k$.
\end{lemma}
\begin{proof}
We will prove the lemma by considering the $k \times k$ sub-matrix, say $ Y $. Recall that $ Y $ is a submatrix  of $ B $ formed by the rows indexed by $R$ and the columns indexed by $C$. For the sake of brevity and abuse of notation, let us denote the restriction of $ B $ to this sub-matrix $ Y :=[ Y_{:1}, \cdots ,  Y_{:k}]$. In what follows, we prove a stronger claim that the submatrix $ Y $ is a rank-$k$ matrix. 

Suppose, for the sake of contradiction that the vectors $\set{ Y_{:1}, \cdots ,  Y_{:k}}$ are linearly dependent. In other words, there exists a vector $ Y_{:i}$ and real numbers $a_1, \cdots, a_k$, not all of which are identically zero, such that 
\[  Y_{:i} = \sum_{j=1, j \neq i}^k a_j  Y_{:j}. \]

From the construction, since Bob inserts a sub-matrix $\gamma \I_k$, we know that 
\begin{align}
	\sum_{j=1}^k ( Y_{j,j} - \gamma)^2 &\leq \|  A -  B  \|_F^2 \leq \Psi . \label{eq:equal} \\
	\sum_{j=1}^k \sum_{p \neq j}  Y_{p,j}^2 &\leq \|  A -  B  \|_F^2 \leq \Psi. \label{eq:neq} 	
\end{align}
From~\eqnref{equal} and choice of $\gamma$,  for all $j$, we have $ Y_{j,j} \geq \Psi^2$. Further,~\eqnref{neq} implies that $ Y_{p,j} \leq \sqrt{\Psi} .$  We have 
\[  Y_{i,i} = \sum_{j=1, j \neq i}^k a_j  Y_{i,j} \geq \Psi^2  
\]
{imply that there is an $p \in \set{1,\cdots, k} \backslash \set{i}$ such that}~ $|a_{p}| \geq  \frac{\Psi^2 }{k\sqrt{\Psi}}.$ 

Let  $\widehat{i}$ be the index in $\set{1,\cdots, k} \backslash \set{i}$ for which $|a_{\widehat{i}} |$ attains the maximum value. We have $|a_{\widehat{i}} Y_{\widehat{i},\widehat{i}}| \geq |a_{\widehat{i}}| \Psi^2$ and  $|a_j  Y_{\widehat{i},j}| \leq |a_{\widehat{i}}| \sqrt{\Psi}$.
Now consider the $\widehat{i}$-entry of $ Y_{:i}$. Note that $\widehat{i} \neq i$. Since $\Psi$ depends quadratically on $m$ and $\tau$, we have
\[ \left|  \sum_{j=1, j \neq i}^k a_j  Y_{\widehat{i},j}  \right|  \geq | a| ( \Psi^2 - k \sqrt{\Psi}  ) \geq  ( \Psi^2 - k \sqrt{\Psi}  ) \frac{\Psi^2 }{k\sqrt{\Psi}} > \sqrt{\Psi} .\] This is a contradiction because for $p \neq j$, $ Y_{p,j} \leq \sqrt{\Psi}$ (\eqnref{neq}). This finishes the proof.
\end{proof}

For the sake of brevity, let $ V _{:1}, \cdots,  V _{:k} $ be the columns of $ B (C)$ and $\widetilde{ V }_{:1}, \cdots, \widetilde{ V }_{:k}$ be the restriction of these column vectors to the rows $a+1, \cdots, m$. In other words, vectors $\widetilde{ V }_{:1}, \cdots, \widetilde{ V }_{:k}$ are the column vectors corresponding to the columns in $M$. We showed in Lemma~\ref{lem:independent} that the columns  $ B (C)$ spans the column space of $ B $. We can assume that the last $n-a$ columns of $ B $ are all zero vectors because $ B $ is a rank-$k$ matrix. We can also assume without any loss of generality that, except for the entries in the row indexed by $R$, all the other entries of $ B (C)$ are zero. This is because we have shown in~\lemref{independent}, we showed that the submatrix of $ B (C)$ formed by rows indexed by $R$ and columns indexed by $C$ have rank $k$. 

Now any row $i$ of $ B $ can be therefore represented as $\sum \eta_{i,j}  V _{:j}$, for real numbers $\eta_{i,j}$, not all of which are identically zero. The following lemma proves part~(ii) of our proof idea. For

\begin{lemma} \label{lem:a}
 Let $ V _{:1}, \cdots,  V _{:k} $ be as defined above. Then  column $i$ of $ B $ can be written as linear combination of real numbers $\eta_{i,1}, \cdots \eta_{i,k}$ of the vectors $ V _{:1}, \cdots,  V _{:k} $ such that, for all $j$ and $i \in R$, $ \eta_{i,j}^2 \leq 4/\Psi^3 $.
\end{lemma}
\begin{proof}
Let $ M _{:1}, \cdots  M _{:a}$ be the columns of $ M $, where $ M $ is the $(W-a )\times a$ submatrix of the matrix $\widetilde{ A }$ corresponding to the input of Alice. We have
\begin{align*}
\Psi & \geq \|  A -  B  \|_F^2  \sum_{i=1}^k (\gamma -  V _{r(a/k)+i,i})^2 + \sum_{i=1}^k \sum_{j \neq i}  V _{r(a/k)+i,j}^2 + \sum_{i=1}^k \|  M _{:r(a/k) + i} - \widetilde{ V }_{:i} \|^2  \nonumber \\
	&\quad + \sum_{i \notin R} \sum_{j=1}^k \paren{ \eta_{i,j}  V _{ra/k + j,j} + \sum_{j' \neq j} \eta_{i,j'}  V _{ra/k+j,j'}  }^2 + \sum_{i \notin R} \left\|  M _{:i} - \sum_{j=1}^k \eta_{i,j} \widetilde{ V }_{:j} \right\|^2.
\end{align*}

As in the proof of~\lemref{independent}, we have $| V _{r(a/k)+i,j}^2 | \leq \sqrt{\Psi}$ and $| V _{r(a/k)+i,i}| \geq \Psi^2$. Let $ {j}_i$ be the index such that $|\eta_{i, j_i} |$ is the maximum. Then the above expression is at least $| \eta_{i,j_i}|^2 (\Psi^2  - k \sqrt{\Psi} )^2 \geq | \eta_{i,j_i}|^2 \Psi^4 /4$. Since this is less than $\Psi $, the result follows from the definition of $j_i$.
 \end{proof}

We can now complete the proof. First note that since $ M $ is a signed matrix, each $\widetilde{ V }_i$ in the third term of the above expression is at least $\sqrt{\Psi}$. Therefore, for all $i \notin S$ and all $j$ 
$$ \left| \sum_{j=1}^k \eta_{i,j} \widetilde{ V }_{:j} \right| \leq \frac{4k\Psi^{1/2} }{\Psi^{3/2}} = \frac{4k}{\Psi}. $$ 

As $ M _{:i}$ is a sign vector and if $\tau=O(m+n) = O(m)$, this implies that 
\begin{align*}
 \sum_{i \notin R} \left\|  M _{:i} - \sum_{j=1}^k \eta_{i,j} \widetilde{ V }_{:j} \right\|^2 &\geq \sum_{i \notin R} \|  M _{:i} \|^2 \paren{1 - \frac{4k}{\Psi} }  \geq O(  (100)^\theta Wd /\ell) - O(100^\theta a) \\
\sum_{i=1}^k \left\|  M _{:r(a/k) + i} - \widetilde{ V }_{:i} \right\|^2 &= \sum_{i=1}^k \sum_{j=1}^{W-a } (  M _{j,r(a/k) + i} - (\widetilde{ V }_i)_j )^2\leq  \frac{100^\theta k(W-a )}{5}  + O(100^\theta a) 
 \end{align*}
 
 Now, since there are in total $k(n-a )$ entries in the submatrix formed by the columns indexed by $C$, at least $1- \paren{\frac{1}{10} + \frac{1}{99} +o(1)}$ fraction of the entries have the property that the sign of $ M _{j,ra/k+i}$ matches the sign of $\widetilde{ V }_{j,i}$. Since ${\mathsf{ind}}$ is in one of the columns of $ M _{:ra/k+1}, \cdots  M _{:ra/k+k}$, with probability at least $1- \paren{\frac{1}{10} + \frac{1}{99} +o(1)}$, if Bob outputs the sign of the corresponding entry in $ B $, then Bob succeeds in solving $\mathsf{AIND}$. This gives a lower bound of $\Omega((W-a )a) =\Omega(Wk \ell/\eta)$ space. 
\end{proof}


\end{appendix}


\end{document}